\documentclass[conference]{IEEEtran}

\usepackage{ifthen}
\newboolean{preprint}
\setboolean{preprint}{false}

\ifthenelse{\boolean{preprint}}{
\usepackage[round]{natbib}
\usepackage{fullpage}
\usepackage{libertine}
\newcommand{\figwidth}{0.55\textwidth}
}{
\newcommand{\figwidth}{.8\linewidth}

}

\usepackage[utf8]{inputenc} %
\usepackage[T1]{fontenc}    %
\usepackage{url}            %
\usepackage{booktabs}       %
\usepackage{amsfonts}       %
\usepackage{nicefrac}       %
\usepackage{microtype}      %
\usepackage{xcolor}         %
\usepackage{natbib}
\usepackage{inconsolata}
\renewcommand{\cite}{\citep}

\usepackage{amsthm}
\usepackage{amsmath}
\usepackage{amsfonts}   %
\usepackage{amssymb}    %
\usepackage{tabularx}
\usepackage{nicefrac}
\usepackage{tcolorbox}
\usepackage{tikz}
\usepackage{thmtools} 
\usepackage{thm-restate}
\usepackage{listings}
\usepackage{subfigure}
\usepackage{algorithm}
\usepackage{algorithmic}

\usepackage[group-separator={,},group-minimum-digits={3}]{siunitx}
\usepackage[pdftex,bookmarksnumbered,bookmarksopen,colorlinks,citecolor={blue!70!black},linkcolor={blue!70!black},urlcolor=blue]{hyperref}

\usepackage[capitalize,noabbrev]{cleveref}
\newcommand{\appref}[1]{Appendix~\ref{#1}}
\Crefname{equation}{Eq.}{Eqs.}
\Crefname{figure}{Fig.}{Figs.}

\theoremstyle{plain}
\newtheorem{theorem}{Theorem}[section]
\newtheorem{lemma}[theorem]{Lemma}

\theoremstyle{definition}
\newtheorem{definition}[theorem]{Definition}

\theoremstyle{remark}

\theoremstyle{definition}
\newtheorem{desideratum}{Desideratum}
\newtheorem{proposition}[theorem]{Proposition}

\newcommand{\tikzxmark}{%
\tikz[scale=0.23] {
    \draw[line width=0.7,line cap=round] (0,0) to [bend left=6] (1,1);
    \draw[line width=0.7,line cap=round] (0.2,0.95) to [bend right=3] (0.8,0.05);
}}
\newcommand{\tikzcmark}{%
\tikz[scale=0.23] {
    \draw[line width=0.7,line cap=round] (0.25,0) to [bend left=10] (1,1);
    \draw[line width=0.8,line cap=round] (0,0.35) to [bend right=1] (0.23,0);
}}

\newcommand{\desbox}[2]{%
    \ifthenelse{\boolean{preprint}}{\vspace{1em}}{}
    \noindent
    \begin{tcolorbox}[
        colback=white,
        colframe=black,
        boxrule=1pt,
        left=0pt,
        right=0pt,
        top=0pt,
        bottom=0pt
    ]
        \begin{desideratum} \label{#1}
            #2
        \end{desideratum}
    \end{tcolorbox}
    \ifthenelse{\boolean{preprint}}{\vspace{.1em}}{}
}
\newcommand{\sR}{\mathbb{R}}
\newcommand{\sD}{\mathbb{D}}
\newcommand{\define}{~\smash{\triangleq}~}
\DeclareMathOperator{\E}{\mathbb{E}}

\ifthenelse{\boolean{preprint}}{
\title{Gaussian DP for Reporting Differential Privacy Guarantees in Machine Learning}
\date{}

\usepackage{authblk}
\author[1]{Juan Felipe Gomez}
\author[2]{Bogdan Kulynych}
\author[3]{Georgios Kaissis}
\author[1]{\\Flavio P. Calmon}
\author[3]{Jamie Hayes}
\author[3]{Borja Balle}
\author[4]{Antti Honkela}

\affil[1]{Harvard University}
\affil[2]{Lausanne University Hospital}
\affil[3]{Google DeepMind}
\affil[4]{University of Helsinki}

}{
\title{Gaussian DP for Reporting Differential Privacy Guarantees in Machine Learning}
\usepackage{authblk}
\author[1]{Juan Felipe Gomez}
\author[2]{Bogdan Kulynych}
\author[3]{Georgios Kaissis}
\author[1]{\\Flavio P. Calmon}
\author[3]{Jamie Hayes}
\author[3]{Borja Balle}
\author[4]{Antti Honkela}

\affil[1]{Harvard University}
\affil[2]{Lausanne University Hospital}
\affil[3]{Google DeepMind}
\affil[4]{University of Helsinki}

}

\begin{document}

\maketitle

\begin{abstract}
    \noindent
    Current practices for reporting differential privacy (DP) guarantees  for machine learning (ML) algorithms such as DP-SGD provide an incomplete and potentially misleading picture. For instance, if only a single $(\varepsilon, \delta)$ is known about a mechanism, standard analyses show that there could exist highly accurate inference attacks against training data records, when, upon a more careful analysis, such accurate attacks do not exist for most practical mechanisms. In this position paper, we argue that using \emph{non-asymptotic} Gaussian Differential Privacy (GDP)  as the primary means of communicating DP guarantees in ML avoids these potential downsides. Using two recent developments in the DP literature: (i) open-source numerical accountants capable of computing the privacy profile and $f$-DP curves of DP-SGD to arbitrary accuracy, and (ii) a decision-theoretic metric over DP representations, we show how to provide non-asymptotic bounds on GDP using numerical accountants, and show that GDP can capture the entire privacy profile of DP-SGD and related algorithms with virtually no error, as quantified by the metric. To support our claims, we investigate the privacy profiles of state-of-the-art DP large-scale image classification, and the TopDown algorithm for the U.S. Decennial Census, observing that GDP fits their profiles remarkably well in all cases. We conclude with a discussion on the strengths and weaknesses of this approach, and discuss which other privacy mechanisms could benefit from GDP.
\end{abstract}

\section{Introduction}
\label{sec:intro}

\begingroup
\renewcommand\thefootnote{}
\footnotetext{This work has been accepted for publication at the IEEE Conference on Secure and Trustworthy Machine Learning (SaTML). The final version will be available on IEEE Xplore.}
\endgroup

\definecolor{bad}{HTML}{CC0000}
\begin{figure*}[t]
    \centering
    \subfigure[Laplace]{
        \centering
        \includegraphics[width=0.266\linewidth]{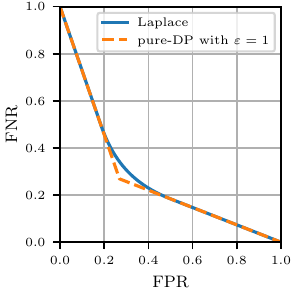}
        \label{fig:laplace}}
    \subfigure[DP-SGD]{
        \centering

        \includegraphics[width=0.28\linewidth]{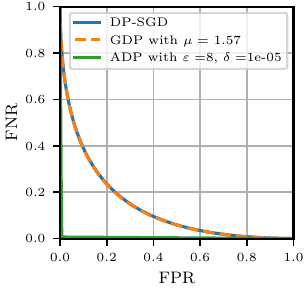}
        \label{fig:de}}
    \subfigure[Regret, \%]{
        \centering
        \raisebox{22mm}{
        \resizebox{0.28\linewidth}{!}{
            \begin{tabular}{lrr}
            & Laplace & DP-SGD \\
            \midrule
            \multicolumn{1}{l}{Concise} \\
            \midrule
            
            $\varepsilon$-DP & \textbf{3.43} & -- \\
            $(\varepsilon, \delta)$-DP & \textbf{3.42} & 21.70 \\
            $\rho$-zCDP & 4.83 & 21.80 \\
            $\mu$-GDP & 3.70 & \textbf{0.10} \\
            
            \midrule
            \multicolumn{1}{l}{Non-concise} \\
            \midrule
            
            $(t, \varepsilon(t))$-RDP & 3.38 & 20.80 \\
            $(\varepsilon, \delta(\varepsilon))$ or $f$ & 0.00 & 0.00 \\
            \end{tabular}
        
        }
        \label{fig:metric_table}}}
     \caption{\small{\emph{Left}: Comparison between the Laplace trade-off curve ($b=1$) and the DP trade-off curve with $\varepsilon = 1$. Higher means more private, hence the pure-DP guarantee is a valid and visually tight bound for Laplace mechanism.  \emph{Middle}: Comparison between a DP-SGD trade-off curve ($\sigma = 9.4, T = \num{2000}, q = 0.33$) from \citet{de2022unlockinghighaccuracydifferentiallyprivate} and a GDP guarantee. This shows that \emph{the GDP bound is tighter for DP-SGD than the $\varepsilon$-DP bound is for Laplace}. \emph{Right}: we quantify the regret from using the DP parameterization over the exact trade-off curve (a measure of ``goodness-of-fit''). Lower means more accurate. We fix $\delta = 10^{-5}$. Although GDP is not universally the best representation (it is not the most accurate for Laplace), GDP is the most accurate concise representation for DP-SGD. We provide technical details in \appref{app:tech_details}.
     }}
    \label{fig:fig1}
\end{figure*}

Ensuring data privacy in machine learning (ML) workflows is crucial, particularly as models trained on sensitive data are increasingly deployed and shared. 
Differential Privacy (DP)~\cite{dwork2006calibrating} has emerged as the gold standard for privacy-preserving ML, offering provable  
guarantees against a broad class of privacy attacks~\cite{salem2023sok}. In principle, any model trained using a DP \emph{mechanism} comes with a formal bound on the amount of information that can be learned about individual training records, regardless of the adversary’s auxiliary knowledge or computational power. 
In the standard variant known as approximate DP (ADP), the strength of the guarantee is controlled by a \emph{privacy budget} parameter $\varepsilon$ and a constant $\delta$. Conventions for setting $\delta$ vary, but it is often set to $ \nicefrac{1}{N^c}$ for $c > 1$, where $N$ is the dataset size~\cite{ponomareva2023how}, or set to be cryptographically small~\cite{vadhan2017complexity}. 

The canonical algorithm for training private deep learning models is DP-SGD \cite{abadi2016deep}, which adds noise to clipped per-example gradients during stochastic optimization. Thanks to its simplicity, DP-SGD is widely adopted and forms the backbone of nearly all state-of-the-art private ML pipelines, including for image classification \cite{de2022unlockinghighaccuracydifferentiallyprivate}, 
and large language model (LLM) fine-tuning~\cite{chua2024mindprivacyunituserlevel, yu2022differentially, lin2023differentially}. The actual privacy protection conferred by DP-SGD 
is most accurately captured by a \emph{privacy profile} $\delta(\varepsilon)$~\cite{balle2018privacy,koskela20b}, i.e., a collection of ADP guarantees. An equivalent and more interpretable view of privacy profiles is given by the \emph{trade-off function} in $f$-DP \cite{dong2022gaussian}, which characterizes the achievable false positive and false negative rates of a worst-case membership inference attack (MIA) aiming to determine whether a specific sample was part of the training dataset. To this end, the past decade has seen significant progress in analyzing the privacy properties of DP-SGD, and led to the development of powerful numerical accountants \cite{koskela2021computingdifferentialprivacyguarantees, gopi2021numerical, alghamdi2023saddle, doroshenko2022connectdotstighterdiscrete} that can compute the entire privacy profile or trade-off function for complex DP workflows. 

Ideally, when reporting the privacy guarantees of DP algorithms such as DP-SGD, we want to report the entire privacy profile or the trade-off function, as it paints a complete picture of the algorithm's privacy guarantees. As this is impractical, the DP community has continued to report DP guarantees using a single $(\varepsilon, \delta)$ pair. This choice necessarily has downsides. Most notably, a single $(\varepsilon, \delta)$-DP pair provides particularly pessimistic bounds when converted into interpretable bounds on attack risk~\cite{kulynych2024attack}. 

Recent research on membership inference attacks (MIAs)~\cite{rezaei2021difficulty,carlini2022membership} focuses on bounding true positive rate ($\textsc{tpr} = 1 - \textsc{fnr}$) at low false positive rate ($\textsc{fpr}$), as this limits the adversary's ability to confidently detect any data record's membership.
In \cref{fig:de} we illustrate the guarantees of a DP-SGD instance (blue line) which ensures that the true positive rate of an inference attack at a false positive rate of $10\%$ is at most $61\%$.
If we only knew the respective $(\varepsilon, \delta)$-DP guarantee at $\delta = 10^{-5}$ (green line), it would appear as if the true positive rate were bounded by $99.95\%$, turning the guarantee into an almost meaningless one.
We provide a more detailed visual representation of the MIA bounds in the low \textsc{fpr} regime in this case in \appref{app:gaussian_app}.

Moreover, $\varepsilon$ values are incomparable if they are computed at different $\delta$. For instance, a realistic mechanism with $\varepsilon = 8$ at $\delta = 10^{-9}$ can be more private in every aspect than a mechanism with $\varepsilon = 6$ at $\delta = 10^{-5}$ (see \cref{tab:epsilon_mu}). As standard conventions set $\delta$ as a function of the dataset size, incomparability is likely across different settings. This problem can be avoided by using the entire privacy profiles of the compared mechanisms~\cite{kaissis2023bounding}. 
These issues demonstrate the need for more sophisticated privacy reporting that uses more information from the privacy profile.

We postulate that a useful method for reporting privacy guarantees in privacy-preserving ML needs to adhere to three desiderata: (1) it should consist of one or two scalar parameters like $(\varepsilon, \delta)$, with one of the parameters having the semantics of a privacy budget like $\varepsilon$, (2) we should be able to compare mechanisms by the budget parameter, and (3) the parameters should accurately represent privacy guarantees for practical mechanisms such as DP-SGD. To understand which DP representations satisfy these requirements, we limit ourselves to common concise parameterizations that satisfy the desiderata (1) and (2). These are ADP (if we assume a fixed $\delta$), zero-concentrated DP (zCDP)~\cite{dwork2016concentrated,bun2016concentrated}, and Gaussian DP~\cite{dong2022gaussian}. To quantify their adherence to (3), we re-purpose a recent decision-theoretic metric between DP mechanisms~\cite{kaissis2024beyond} to measure \emph{regret} of using a given privacy representation instead of the complete privacy profile or the trade-off function, and empirically evaluate their fit in practical deployments.

This comparison is challenging as (a) DP-SGD does not admit simple analyses in terms of zCDP, and (b) the standard analyses of DP-SGD in terms of GDP are asymptotic, which results in \emph{optimistic}\footnote{DP reports ``worst-case'' pessimistic bounds on privacy loss. While optimistic ``best-case'' bounds are interesting, they are not the focus of DP or this work.}
estimates of privacy loss~\cite{gopi2021numerical}. To address (a), we use a numeric approach to find the optimal zCDP guarantee from a set of R\'enyi DP guarantees~\cite{mironov2017renyi} obtained using the standard moments accounting procedure~\cite{abadi2016deep,mironov2017renyi}. For (b), we propose a new way to obtain a \emph{pessimistic} bound on GDP based on numerical accounting. This enables us to compare these representations on equal terms.

Empirically, we find that various practical deployments of DP machine learning algorithms are almost exactly characterized by a pessimistic, non-asymptotic $\mu$-GDP guarantee. In particular, we observe this behaviour for DP 
large-scale image classification models~\cite{de2022unlockinghighaccuracydifferentiallyprivate} and, beyond ML, the TopDown algorithm for the U.S.\@ Decennial Census~\cite{abowd2022census}. 

As an illustration, in \cref{fig:fig1} we show that a pessimistic, non-asymptotic GDP guarantee characterizes the behavior of DP-SGD \emph{more precisely} than $\varepsilon$-DP characterizes the privacy guarantees of the standard Laplace mechanism. This can be seen in \cref{fig:fig1}(c), where the regret from using $\varepsilon$-DP to represent the trade-off curve of the Laplace mechanism is $3.43\%$, and the regret from using $\mu$-GDP to represent the trade-off curve of DP-SGD is $0.1\%$. This  can also be seen in \cref{fig:fig1}(a,b) where the dotted orange line is a tighter lower-bound to the blue for DP-SGD compared to Laplace. Thus, GDP satisfies all the desiderata for a useful privacy parameterization for many realistic cases.

Based on these observations, \textbf{we call the DP community to move beyond $\varepsilon$ at fixed $\delta$ as the standard for reporting privacy guarantees for algorithms that admit tight analyses in terms of the privacy profile or the trade-off curve, such as DP-SGD. Instead, we propose converting the privacy profile to a \emph{pessimistic, non-asymptotic}, $\mu$-GDP guarantee. We further propose to optionally test whether it provides an accurate representation using the decision-theoretic regret metric, and treating $\mu$-GDP as complete privacy representation if the test passes.} 
When GDP is a ``good fit'' according to the regret metric---which is the case for many realistic instances in privacy-preserving ML---it offers a concise single-parameter yet practically complete representation of privacy guarantees, enabling comparability across settings and precise characterizations of attack risk. In the paper, we provide a method for obtaining such a $\mu$-GDP guarantee using accountants, and a method to test if the GDP guarantee is accurate. 
A Python package which enables to perform these steps is available at:
\begin{center}
    \url{https://github.com/interpretable-dp/gdpnum}
\end{center}

\section{Technical Background and Tools}
\label{sec:preliminaries}
In this section, we overview the background and tools needed to understand our position. This section was written for readers with technical familiarity with DP terminology.
A more detailed overview can be found in \appref{app:detailed}.
Let $S \in \sD^N$ denote a dataset with $N$ individuals over a data record space $\sD$. We use $S \simeq S'$ to denote when two datasets are neighbouring under an (arbitrary) neighbouring relation. Let $M$ denote a randomized algorithm (or mechanism) that maps datasets to probability distributions over some output space. Let $\Theta$ denote the output space, and a specific output as $\theta \in \Theta$. In a slight abuse of notation, we use $M(S)$ to denote both the probability distribution over $\Theta$ and the underlying random variable.

\subsection{Classical Differential Privacy} 
\begin{definition}[\citealp{dwork2006calibrating,dwork2014algorithmic}]
A mechanism $M: \sD^N \rightarrow \Theta$ satisfies $(\varepsilon, \delta)$-DP if for any measurable $E \subseteq \Theta$ and $S \simeq S'$, we have
$ \Pr[M(S) \in E] \leq e^\varepsilon \Pr[M(S') \in E] + \delta$. We say that the mechanism satisfies \emph{pure DP} if $\delta = 0$ and \emph{approximate DP} (ADP) otherwise.
\end{definition}
Most DP algorithms satisfy a continuum of approximate DP guarantees, hence we say that a mechanism $M$ has a privacy profile $\delta(\varepsilon)$ if for every $\varepsilon \in \mathbb{R}$, it is $(\varepsilon, \delta(\varepsilon))$-DP.

\subsection{DP-SGD}\label{app:dpsgd}
DP-SGD~\cite{song2013stochastic,bassily2014differentially,abadi2016deep} is a differentially private adaptation of the Stochastic Gradient Descent (SGD). A core building block of DP-SGD is the \emph{subsampled Gaussian mechanism}:
\begin{equation}
    M_p(S) = g(\mathsf{Subsample}_p(S)) + Z,
\end{equation}
where $Z \sim \mathcal{N}(0, \sigma^2)$, and $\mathsf{Subsample}_p(S)$ denotes \emph{Poisson} subsampling of the dataset $S$, which includes any element of $S$ into the subsample with probability $p \in [0, 1]$. 

DP-SGD is an iterative algorithm which applies subsampled Gaussian mechanism to the whole training dataset multiple times, where each time the query function $g(\cdot)$ corresponds to computing the loss gradient on the Poisson-subsampled batch of training data examples, and clipping the per-example gradients to ensure their bounded $L_2$ norm. DP guarantees depend on the the sampling rate $q = B/N$ (where $B$ is the expected batch size under Poisson sampling and $N$ is the dataset size), the number of iterations $T$, and the noise parameter $\sigma$.

\subsection{Differential Privacy Variants}
We also consider DP variants based on R\'enyi divergence.
\begin{definition}[\citealp{mironov2017renyi,bun2016concentrated}]
A mechanism $M(\cdot)$ satisfies $(t, \varepsilon(t))$-RDP if for all $S\simeq S'$ the R\'enyi divergence of order $t$ from $M(S)$ to $M(S')$ is bounded by $\varepsilon(t)$. See \appref{app:rdp} for the definition of the R\'enyi divergence. 
The mechanism satisfies $\rho$-zCDP if it satisfies $(t,\rho \: t)$-RDP for all $t \geq1 $ given $\rho \geq 0$. 
\end{definition}

DP can be equivalently characterized via a constraint on the success rate of a hypothesis test~\cite{wasserman2010statistical, kairouz2015composition, dong2022gaussian}. Given datasets $S \simeq S'$ and mechanism $M$, an adversary aims to determine if a given output $\theta \in \Theta$ came from $M(S)$ or $M(S')$ 
via running a binary hypothesis test $H_0: \theta \sim M(S), \quad H_1: \theta \sim M(S')$,
where the test is modelled as a test function $\phi: \Theta \rightarrow [0,1]$ which associates a given output $\theta$ to the probability of the null hypothesis $H_0$ being rejected.

We can analyze this hypothesis test in terms of the trade-off between the attainable \emph{false positive rates} (FPR) $\alpha_\phi \define \smash{\E_{\theta \sim M(S)}}[\phi(\theta)]$ and \emph{false negative rates} (FNR) $\beta_\phi \define 1 - \smash{\E_{\theta \sim M(S')}}[\phi(\theta)]$. 
This can be done via the \emph{trade-off curve}, a function that outputs the lowest achievable FNR at any given FPR $\alpha$: $T(M(S), M(S'))(\alpha) \define \inf_{\phi:~\Theta \rightarrow [0, 1]}\{ \beta_{\phi} \mid \alpha_\phi \leq \alpha \}$. This trade-off curve forms the basis of a more general version of DP called $f$-DP.

\begin{definition}[\citealp{dong2022gaussian}]
A mechanism $M$ satisfies $f$-DP if for any $S \simeq S'$ and $\alpha \in [0,1]$, we have that $T(M(S), M(S'))(\alpha) \geq f(\alpha)$.
Note that a \emph{valid} trade-off curve $f: [0, 1] \rightarrow [0, 1]$ must be non-increasing, convex, and upper bounded as $f(\alpha) \leq 1 - \alpha$.
\end{definition}

The $f$-DP notion is more general than DP: a mechanism $M$ is $(\varepsilon, \delta)$-DP iff it satisfies $f$-DP with:
\begin{equation}\label{eq:dp-to-f}
    f_{\varepsilon, \delta}(\alpha) = \max\{0, 1 - \delta - e^\varepsilon \alpha,\ e^{-\varepsilon} (1 - \delta - \alpha)\}.
\end{equation}
Similarly to \cref{eq:dp-to-f}, other representations such as R\'enyi DP and zCDP \emph{induce} a trade-off curve, that we call the \emph{associated trade-off curve} of a representation (see \appref{app:trade-off-curves-repr} for details). 

Moreover, it turns out that an $f$-DP trade-off curve is equivalent to a privacy profile:
\begin{theorem}[\citealp{dong2022gaussian}]\label{cor:profile_to_f}
    A mechanism $M$ satisfies $(\varepsilon, \delta(\varepsilon))$-DP iff it is $f$-DP with:
    \begin{equation}
        f(\alpha) = \sup_{\varepsilon \in \mathbb{R}} \max \{0, 1 - \delta(\varepsilon) - e^\varepsilon \alpha, e^{-\varepsilon}(1 - \delta(\varepsilon) - \alpha) \}.
    \end{equation}
\end{theorem}

In practice, the privacy profiles for complex algorithms such as DP-SGD, which involve composition, are computed numerically via algorithms called accountants~\cite[see, e.g.,][]{abadi2016deep,koskela2021computingdifferentialprivacyguarantees,gopi2021numerical,doroshenko2022connectdotstighterdiscrete}. These algorithms compute profiles to accuracy nearly matching the lower bound of a privacy audit where the adversary is free to choose the entire (often pathological) training dataset \cite{nasr2021adversary,nasr2023tight}. Given these results, we can treat the analyses of numerical accountants as exact up to floating-point precision. \cref{cor:profile_to_f} implies that privacy curves $\delta(\varepsilon)$ from numerical accountants can be transformed into trade-off functions, and there exist efficient and practical algorithms for performing such conversions~\cite{kulynych2024attack}.

\subsection{Gaussian Differential Privacy Beyond Asymptotics}
Gaussian Differential Privacy (GDP) is a special case of $f$-DP where the bounding function $f$ is defined by a test to distinguish a single draw from a unit variance Gaussian with zero mean versus one from a unit variance Gaussian with mean $\mu$. The resulting trade-off curve is:
\begin{definition}[\citealp{dong2022gaussian}]
    A mechanism $M$ satisfies $\mu$-GDP iff it is $f_\mu$-DP with:
    \begin{equation}\label{eq:gdp}
        f_\mu(\alpha) = \Phi(\Phi^{-1}(1- \alpha) - \mu),
    \end{equation}
    where $\Phi$ denotes the CDF and $\Phi^{-1}$ the quantile function of the standard normal distribution. 
\end{definition}
The parameter $\mu$ is similar to $\varepsilon$ in standard DP in the sense that it quantifies privacy loss: higher values of $\mu$ correspond to less private algorithms. Although previous work \cite{dong2022gaussian, Bu2020Deep} focused on deriving asymptotic $\mu$-GDP guarantees for ML algorithms such as DP-SGD, in this work we take advantage of the fact that non-asymptotic numerically precise trade-off curves are readily available to compute optimally tight GDP guarantees. Given a mechanism with a trade-off curve $f$, we seek the smallest possible $\mu$ such that the mechanism is $\mu$-GDP. Specifically, we wish to find:
\begin{equation}
    \mu^* = \inf\{\mu \geq 0 ~\mid~ \forall \alpha \in [0, 1]: \: f_{\mu}(\alpha) \leq f(\alpha)\}.\label{eq:mu}
\end{equation}
A similar expression was used by \citet{koskela2023individual}, albeit under a different context. This $\mu^*$ parameter is tight in the sense that there is no $\mu' < \mu^*$ such that the mechanism is $\mu'$-GDP. It turns out that \cref{eq:mu} is particularity simple to solve due to the piecewise-linear structure of trade-off curves generated by numerical accountants. We leave the technical details to the appendix: \appref{app:profiles} discusses accountants in detail and \appref{app:gdp_details} shows how to solve \cref{eq:mu}. We remark that for the numerical accountants used in this work~\cite{doroshenko2022connectdotstighterdiscrete}, we can solve \cref{eq:mu} in microseconds on commodity hardware. Hence, it is easy to take a trade-off curve $f$ from a numerical accountant and convert it to a tight $\mu$-GDP guarantee.

\subsection{Representation Regret: A Metric Over Trade-Off Curves}\label{sec:metric}

The last concept we need is a recently proposed metric over DP mechanisms, which can be equivalently interpreted as a metric over privacy guarantee representations. This metric is based on the hypothesis testing interpretation of DP, and is defined via trade-off functions. 
\begin{definition}[\citealp{kaissis2024beyond}]\label{def:metric}
Given two valid trade-off functions $f, \tilde{f}$, the $\Delta$-divergence from $f$ to $\tilde f$ is:
\begin{equation}\label{eq:delta-div}
    \Delta(f, \tilde{f}) \define \inf \{ \kappa \geq 0 ~\mid~ \forall \alpha \in [0, 1]: f(\alpha + \kappa) - \kappa \leq \tilde{f}(\alpha)\}.
\end{equation}
Moreover, the symmetrized $\Delta$-divergence is a metric over trade-off curves and is defined as:
\begin{equation}\label{eq:symm-delta-div}
    \Delta^{\leftrightarrow}(f, \tilde f) \define \max\{\Delta(f, \tilde f), \Delta(\tilde f, f)\}.
\end{equation}
\end{definition}
Due to a classical result by \citet{blackwell,dong2022gaussian}, we know that if $f(\alpha) \leq \tilde f(\alpha)$ for all $\alpha \in [0,1]$, then 
$\smash{\tilde f}$ is uniformly more private than $f$.
Intuitively, $\Delta(f, \smash{\tilde f})$ quantifies how far down and left one needs to shift $f$ so that $\smash{\tilde f}$ is uniformly more private. 
If $\Delta(f, \smash{\tilde f})$ is small, this implies that $f, \smash{\tilde f}$ are close. 
In our context, $\smash{\tilde f}$ corresponds to the trade-off curve associated with a pessimistic DP guarantee such as $(\varepsilon, \delta)$ or $\mu$-GDP, and $f$ corresponds to the exact trade-off curve of a mechanism. We hence refer to $\Delta(f, \smash{\tilde f})$ as the \emph{regret} of reporting the pessimistic bound $\smash{\tilde f}$ over the exact numerical trade-off curve $f$. See \cref{fig:regret} for an illustration.

\begin{figure}
    \centering
    \includegraphics[width=\figwidth]{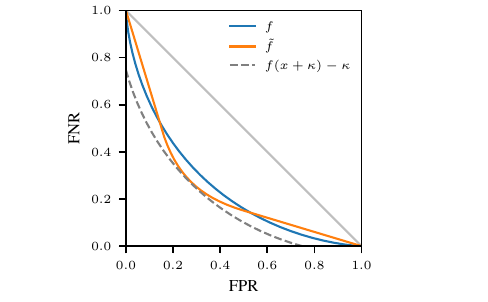}
    \caption{Illustration of the \citet{kaissis2024beyond} regret metric between two mechanisms which satisfy $f$-DP and $\tilde f$-DP, respectively. The metric $\Delta(f, \tilde f)$ is the smallest $\kappa \geq 0$ such that $f(\alpha + \kappa) - \kappa$ dominates $\tilde f$. }
    \label{fig:regret}
\end{figure}

Similar to \cref{eq:mu}, the structure of the trade-off curves from numerical accountants make computing $\Delta(f, \smash{\tilde f})$ practical and easy to implement, and can be done in milliseconds on commodity hardware. We leave the details of the numerics to \appref{app:gdp_details}. 

In \cref{sec:proposed-framework} we additionally provide an operational interpretation of the values of regret in terms of risk of standard attacks against data privacy.

\section{Desiderata for Reporting Privacy}
\label{sec:desiderata}

Building on the tools detailed in \cref{sec:preliminaries}, we argue that the DP community is well-positioned to rethink its conventional methods for reporting privacy guarantees in machine learning and especially DP-SGD. 
With the development of numerical accountants \cite{koskela2021computingdifferentialprivacyguarantees, gopi2021numerical, alghamdi2023saddle, doroshenko2022connectdotstighterdiscrete} capable of computing trade-off curves to arbitrary accuracy \cite{kulynych2024attack}, and the introduction of metrics for quantifying the distance between two trade-off curves \cite{kaissis2024beyond}, the tools today far exceed those present when the current standards (reporting $\varepsilon$ at sufficiently small $\delta$) were established. 
In this section, we identify key criteria that any effective reporting standard should satisfy. We then show the limitations of the current approaches and present a more robust alternative. \cref{tab:desiderata} provides a summary.

\begin{table}[tb]
    \caption{Comparison of DP variants and their match to desiderata.}
    \begin{tabularx}{\columnwidth}{Xccc}
           & \scriptsize D1 & \scriptsize D2 & \scriptsize D3 \\
    Method & Concise & Ordered & Accurate \\
    \midrule
    $\varepsilon$-DP & \tikzcmark & \tikzcmark & \tikzxmark \\
    $(\varepsilon, \delta)$-DP & \tikzcmark & \tikzxmark & \tikzxmark \\
    $ (\varepsilon, \delta(\varepsilon))$ or $f$ & \tikzxmark & \tikzxmark & \tikzcmark \\
    $(t, \varepsilon(t))$-RDP & \tikzxmark & \tikzxmark & \tikzxmark \\
    $\rho$-zCDP & \tikzcmark & \tikzcmark & \tikzxmark \\
    $\mu$-GDP & \tikzcmark & \tikzcmark & \tikzcmark
    \end{tabularx}
    \label{tab:desiderata}
    \vspace{-1em}
\end{table}

\vspace{.5em}
\desbox{des:1}{
        Concise (one- or two-parameter) representation of privacy guarantees, 
        with one of the parameters having the interpretation of a ``privacy budget''.
}
\noindent This is a common goal in practice. For example, pure DP \cite{dwork2006calibrating} provides a single, clear, and worst-case bound $\varepsilon$ on how much any individual’s data can influence an output of a mechanism. Moreover, the $\varepsilon$ parameter increases under composition, which motivated the concept of a \emph{privacy budget} being expended. This intuition is preserved with other privacy definitions such as zCDP and GDP, the parameters of which also increase under composition, though it is worth mentioning that optimal composition for certain representations (e.g. zCDP, GDP) are much simpler than that in pure and approximate DP~\cite{kairouz2015composition}. As a result, these guarantees are not only easy to interpret but also straightforward to report and manage. A profile approach, where one reports either the full privacy profile $\delta(\varepsilon)$, trade-off curve $f$, or RDP curve $\varepsilon(t)$, does not satisfy this property.

\vspace{.5em}
\desbox{des:2}{
    The strength of privacy guarantees can be ordered based on the ordering of the parameters.
}
\noindent Let $\gamma, \gamma' \in \mathbb{R}$ denote privacy budget parameters for two mechanisms $M, M'$. Desideratum 2 says that if $\gamma \leq \gamma'$, then $M$ is more private than $M'$. 

Although there exist different approaches to compare privacy-preserving mechanisms
~\cite[see, e.g.,][]{chatzikokolakis2019comparing}, we use the recent approach by \citet{kaissis2024beyond} which establishes the equivalence between comparing mechanisms by their trade-off curve or privacy profile and the standard statistical notion of experiment comparison known as the Blackwell order~\cite{blackwell}. According to this approach, \cref{des:2} holds for the single-parameter definitions. In general, it does not hold for the two-parameter families---approximate DP and RDP---as it is possible to choose $(\varepsilon, \delta)$, $(\varepsilon', \delta')$ such that mechanism $M$ is neither uniformly more or less private than $M'$~\cite{kaissis2024beyond}. 

\vspace{.5em}
\desbox{des:3}{
    The definition accurately represents privacy guarantees of common practical mechanisms with low regret.
}
\noindent Out of the standard approaches, the only information-theoretically complete representations of privacy guarantees for \emph{all} mechanisms are the full privacy profile $\delta(\varepsilon)$ and the trade-off curve $f$. Unfortunately, these representations do not satisfy \cref{des:1}.  If we want a compact representation, we must lose representational power for some mechanisms. This desideratum states that we should \emph{not} lose representational power for the commonly deployed mechanisms in practice. 

The trade-off curve associated with a single $(\varepsilon, \delta)$ pair does not approximate well many practical mechanisms in machine learning, as we demonstrate in \cref{fig:fig1} and \cref{sec:applicability-gdp}. R\'enyi-based definitions---RDP and zCDP---are also known to not be able to precisely capture the trade-off curves~\cite{balle2020hypothesis,asoodehconversion, zhu2022optimalaccountingdifferentialprivacy}. 
We demonstrate this in \cref{fig:fig1}(c), where we show that the numerical accountants and GDP yield tighter characterizations than zCDP and the entire RDP curve $\varepsilon(t)$.
Thus, out of the parameterizations in \cref{sec:preliminaries}, only GDP and the profiles satisfy \cref{des:3}.

\section{Proposed Framework for Reporting Privacy}
\label{sec:proposed-framework}

\begin{figure*}[t]
\centering
\begin{minipage}[t]{0.48\linewidth}
\begin{algorithm}[H]
\caption{Reporting pessimistic, non-asymptotic $\mu$-GDP}
\label{alg:tight_gdp}

\begin{algorithmic}[1]

\STATE Compute trade-off function $f$ via numerical accountants

\STATE Obtain the tight GDP guarantee:
\[
\mu^* \leftarrow \inf\{\mu \geq 0 ~\mid~ \forall \alpha: \: f_{\mu}(\alpha) \leq f(\alpha)\}.
\]

\STATE Evaluate regret (optional):
\[
    \Delta \leftarrow \inf\{\kappa \geq 0 ~\mid~ \forall \alpha: \: f(\alpha + \kappa) - \kappa \leq f_{\mu^*}(\alpha) \}
\]

\STATE \textbf{return} $\mu^*$, $\Delta$
\end{algorithmic}
\end{algorithm}
\end{minipage}
\hfill
\begin{minipage}[t]{0.48\textwidth}
\vspace{1em}
\begin{lstlisting}[language=Python, 
    label={lst:gdpnum_impl},
    basicstyle=\footnotesize\ttfamily,
    keywordstyle=\color{blue}\bfseries,
    commentstyle=\color{green!50!black}\itshape,
    stringstyle=\color{red},
    showstringspaces=false,
    breaklines=true,
    backgroundcolor=\color{white}]
import gdpnum

# Example for DP-SGD.
accountant = gdpnum.CTDAccountant()
data_loader = ...
for mini_batch in data_loader:
    ...
    accountant.step(noise_multiplier=1.0, 
                    sample_rate=0.001)

# Computing mu and regret
mu, regret = accountant.get_mu_and_regret()
\end{lstlisting}
\end{minipage}
\caption{Procedure for reporting the pessimistic, non-asymptotic $\mu$-GDP guarantee (left), and the corresponding instantiation using our Python library (right).}
\label{fig:gdp-procedure}
\end{figure*}

Given the discussion in \cref{sec:desiderata}, we propose the following approach to reporting privacy guarantees. 

\paragraph{Reporting pessimistic, non-asymptotic $\mu$-GDP}
We propose the following procedure for computing the pessimistic, non-asymptotic GDP guarantee:
\begin{enumerate}
    \item Compute the trade-off function $f$ via open-source numerical accountants.
    \item Obtain a non-asymptotic tight $\mu$-GDP guarantee by solving \cref{eq:mu}. The resulting $\mu$ can always be reported as a valid privacy bound.
    \item Optionally, in order to evaluate the accuracy of the $\mu$-GDP bound, evaluate the regret using \cref{eq:delta-div}.
\end{enumerate}
We outline this algorithm at the high level, as well as show the interface using our software in \cref{fig:gdp-procedure}. For technical details, see \appref{app:tech_details}. Note that the entire procedure executes in seconds on commodity hardware.

If regret is $ < 10^{-2}$, $\mu$-GDP provides an essentially complete picture of the privacy guarantees.

\paragraph{Intepreting regret} A natural question that arises from our proposal is what is a good enough value of regret, and why do we suggest $<10^{-2}$? For this, we provide an operational interpretation. Consider \emph{advantage}~\cite{yeom2018privacy,kaissis2024beyond,kulynych2024attack}:
\begin{equation}\label{eq:advantage}
    \eta(f) \define \max_{\alpha \in [0, 1]} 1 - \alpha - f(\alpha),
\end{equation}
i.e., the highest achievable difference between attack $\text{TPR}=1-\text{FNR}$ and $\text{FPR}$, equivalent to the highest achievable normalized accuracy of MIAs. As \citet{cherubin2024closed} showed, not only does this quantity bound MIA accuracy, but also the advantage over random guessing of attribute inference~\cite{yeom2018privacy} and record reconstruction~\cite{balle2018privacy} attacks.
\begin{restatable}{proposition}{advantage}
    \label{stmt:delta-as-tv}
    For any two valid trade-off curves $f, \tilde f$, we have that:
    \begin{equation}
        |\eta(f) - \eta(\tilde f)| \leq 2 \Delta^{\leftrightarrow}(f, \tilde f).
    \end{equation}
\end{restatable}
We provide the proof in \appref{app:metric}.
Thus, the regret threshold of $10^{-2}$ ensures that the highest advantage of inference attacks is pessimistically over-reported by \emph{at most $2$ percentage points}.

Additionally, we present empirical results in \appref{app:plots} that show that, on both standard and log-log scales, the $\mu$-GDP trade-off curve closely follows the original $f$ up to numeric precision for different instantiations of DP when the regret is $<10^{-2}$.

\paragraph{Fallbacks when GDP is not a good representation}\label{sec:tier2}
If regret from using GDP is high or the mechanism cannot satisfy GDP (see \cref{sec:non-uses}), we propose that the practitioners always report the tightest privacy guarantee available, e.g., the tabulation of the privacy profile or a trade-off curve, or the $\rho$-zCDP parameter, in addition to the standard practice of reporting $\varepsilon$ at fixed $\delta$.

\begin{table*}[t]
    \centering
    \caption{\textbf{Unlike $\varepsilon$ with data-dependent values of $\delta$, reporting $\mu$ enables comparisons of mechanisms in terms of privacy guarantees across settings and datasets.} The table shows the before and after comparison of Table 1 from \citet{de2022unlockinghighaccuracydifferentiallyprivate} using our proposed approach, i.e., reporting a conservative $\mu$-GDP guarantee computed with numeric accounting.
    The regret of reporting GDP over the \emph{full privacy profile or the full trade-off curve} is less than $10^{-3}$ (see \appref{app:plots}).}\label{tab:de_et_al}
    \vspace{.5em}
    \resizebox{\textwidth}{!}{
    \begin{tabular}{c c}
        Before & After \\
        \begin{tabular}{l c c c c c c c}
            \hline
            \textbf{Dataset} & \textbf{Pre-Training} & 
            \multicolumn{5}{c}{\textbf{Top-1 Accuracy (\%)}} \\ 
            \hline
            & & $\varepsilon=1$ & $\varepsilon=2$ & $\varepsilon=4$ & $\varepsilon=8$ & $\delta$ & \\ 
            \hline
            CIFAR-10 & -- & 56.8 & 65.9 & 73.5 & 81.4 & $10^{-5}$ \\ 
            ImageNet & -- & -- & -- & -- & 32.4 & $8 \cdot 10^{-7}$ \\ 
            \hline
            CIFAR-10 & ImageNet & 94.7 & 95.4 & 96.1 & 96.7 & $10^{-5}$ \\ 
            CIFAR-100 & ImageNet & 70.3 & 74.7 & 79.2 & 81.8 & $10^{-5}$ \\ 
            \hline
            ImageNet & JFT-4B & 84.4 & 85.6 & 86.0 & 86.7 & $8 \cdot 10^{-7}$ \\ 
            Places-365 & JFT-300M & -- & -- & -- & 55.1 & $8 \cdot 10^{-7}$ \\ 
            \hline
        \end{tabular}
        &
        \begin{tabular}{l c c c c c c c}
            \hline
            \textbf{Dataset} & \textbf{Pre-Training} & 
            \multicolumn{5}{c}{\textbf{Top-1 Accuracy (\%)}} \\ 
            \hline
            & & $\mu=0.21$ & $\mu = 0.39$ & $\mu = 0.72$ & $\mu = 1.3$ &\\ 
            \hline
            CIFAR-10 & -- & 56.8 & 65.9 & 73.5 & 81.4 \\ 
            ImageNet & -- & -- & -- & -- & 32.4 \\ 
            \hline
            CIFAR-10 & ImageNet & 94.7 & 95.4 & 96.1 & 96.7 \\ 
            CIFAR-100 & ImageNet & 70.3 & 74.7 & 79.2 & 81.8 \\ 
            \hline
            ImageNet & JFT-4B & 84.4 & 85.6 & 86.0 & 86.7 \\ 
            Places-365 & JFT-300M & -- & -- & -- & 55.1 \\ 
            \hline
        \end{tabular}
    \end{tabular}
    }
\end{table*}

\section{Example Usage}
\label{sec:applicability-gdp}

In this section, we demonstrate how GDP can accurately represent privacy guarantees for key algorithms.

\paragraph{DP-SGD} %
We empirically observe that the trade-off curve of DP-SGD (when using the add/remove neighbouring relation and exact Poisson subsampling) with practical privacy parameters is close to Gaussian trade-off curve. As an example, we use noise scale $\sigma = 9.4$, subsampling rate $p = 2^{14} / \num{50000}$, and $\num{2000}$ iterations, following the values used by \citet{de2022unlockinghighaccuracydifferentiallyprivate} to train a 40-layer Wide-ResNet to an accuracy of $81.4\%$ on CIFAR-10 under $(\varepsilon = 8, \delta = 10^{-5})$-DP. We observe in \cref{fig:de} that this algorithm is $\mu = 1.57$-GDP with regret $\approx 10^{-3}$, indicating that the $\mu = 1.57$-GDP guarantee captures the privacy properties of the algorithm almost perfectly. See \appref{app:plots} for more figures similar to this one.

Furthermore, we reproduce Table 1 from \citet{de2022unlockinghighaccuracydifferentiallyprivate} in our \cref{tab:de_et_al}, and compare their presentation with a version using our proposed approach side-by-side. 
Crucially, all the privacy parameters $\mu$ are comparable across settings, unlike $\varepsilon$ values which are only comparable when $\delta$ is the same.

We further investigate the regime over which a $\mu$-GDP guarantee fits well for DP-SGD in \cref{fig:dpsgd-sweep}. Darker colors denote a higher number of compositions. We observe that, for fixed noise parameter $\sigma$ (i.e. fixed color in \cref{fig:dpsgd-sweep}) and sampling probability, increasing compositions always leads to a better $\mu$-GDP fit and a lower regret. 
For fixed number of compositions (i.e., fixed darkness of the lines) and sampling rate, the higher the noise parameter $\sigma$ the better is the $\mu$-GDP fit. There is a non-monotonic relation between the sampling rate and regret for fixed noise parameter $\sigma$ and number of compositions. This non-trivial dependence highlights the need for care when summarizing DP-SGD with a $\mu$-GDP guarantee. From \cref{fig:dpsgd-sweep}, however, we observe:

\begin{tcolorbox}[
    colback=white,
    colframe=black,
    boxrule=1pt,
    left=0pt,
    right=0pt,
    top=0pt,
    bottom=0pt
]
    \textbf{Rule of thumb.} Any DP-SGD algorithm run with noise parameter $\sigma \geq 2$ and number of iterations $T \geq 400$ will satisfy a $\mu$-GDP guarantee with regret less than 0.01.  
\end{tcolorbox}

\begin{figure}
    \centering
    \includegraphics[width=\figwidth]{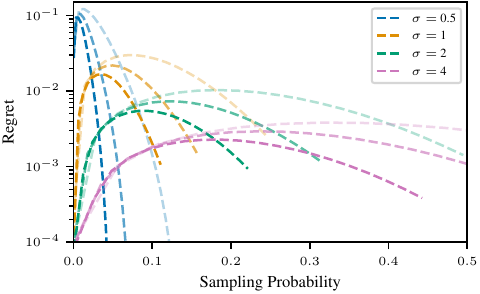}
    \caption{Worst-case regret values as a function of the sampling rate in DP-SGD for various choices of noise parameter $\sigma$ and compositions. We sweep over $T = \{400, 1000,2000\}$ compositions, with darker lines indicating higher composition numbers. }
    \label{fig:dpsgd-sweep}
\end{figure}

\begin{figure*}[t]
    \centering
    \subfigure[TopDown]{
        \centering
        \includegraphics[width=0.32\linewidth]{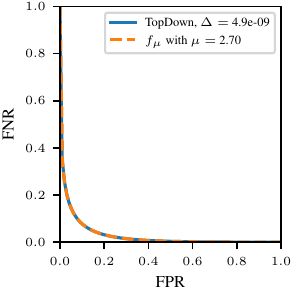}
        \label{fig:app_third_row_errors_topdown}}
    \subfigure[Randomized Response ($\varepsilon = 1$)]{
        \centering
        \includegraphics[width=0.32\linewidth]{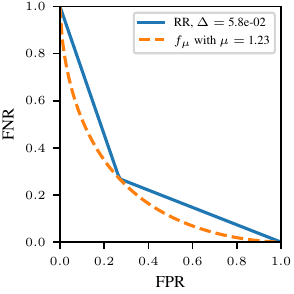}
        \label{fig:app_third_row_errors_rr}}
     \caption{Numerically evaluated trade-off curves and the best conservative $\mu$-GDP bounds for (a) The TopDown algorithm; and (b) Randomized Response.}
    \label{fig:large_combined_figure}
\end{figure*}

\paragraph{Top-Down algorithm}
We replicate the results from \citet{su20242020unitedstatesdecennial}, which reanalysed the privacy accounting in the TopDown algorithm using $f$-DP, according to the privacy-loss budget allocation
released on August 25, 2022 by the US Census Bureau. Their custom accounting code takes $>9$ hours on  96×2GB virtual CPUs. We (1) show that this accounting can be done in a few seconds on a commercial laptop and (2) that the TopDown algorithm is tightly characterized by GDP, achieving $\mu = 2.702$-GDP. See \cref{fig:large_combined_figure} (middle).

\paragraph{Other algorithms}
Multiple practical DP algorithms for deep learning~\cite{kairouz2021practical}, synthetic data generation~\cite{lin2023differentially}, privacy-preserving statistical modelling~\cite{kulkarni2021differentially,rho2022differentially,raisa2024noise}, are based on the composition of simple Gaussian mechanisms. For such algorithms no accounting machinery is needed: GDP can be directly analyzed and reported.

\paragraph{Practical considerations}
\label{sec:practical-considerations}

Typical DP machine learning papers report results with $(\varepsilon, \delta)$-DP with small integer powers of $0.1$ as $\delta$. In \cref{tab:epsilon_mu}, we provide a conversion table between $\mu$-GDP and $(\varepsilon, \delta)$-DP with suggested replacements for commonly used values. For a general intuition on mapping $\varepsilon$ values to $\mu$ values, at $\delta = 10^{-5}$, $\mu = 0.5$ corresponds to roughly $\varepsilon \approx 2$, $\mu = 1$ corresponds to roughly $\varepsilon \approx 4$, and $\mu = 2$ to $\varepsilon \approx 10$. The table also gives a practical illustration of the difficulty of interpreting $(\varepsilon,\delta)$, as Gaussian mechanism with $(\varepsilon=8,\delta=10^{-9})$ is more private than $(\varepsilon=6,\delta=10^{-5})$.

In \cref{tab:risk}, we also provide a mapping from various $\mu$ values to the notion of risk captured by the points on the $f$-DP trade-off curve. These points can be interpreted in two ways. First, by definition, they show the maximum TPR of any membership inference attack with a given FPR. Second, there is an additional interpretation, recently demonstrated by the subset of the authors~\cite{kulynych2025unifying}: TPR obtained from $f$-DP is a conservative upper bound on the maximum attack success rate of any singling out, attribute inference, or data reconstruction attack whose baseline success rate (i.e., without the adversary observing the mechanism's output) equals the considered FPR.

In the table, we also show $\eta$-TV privacy which corresponds to different values of $\mu$, which can also be interpreted in two corresponding ways. First, it shows maximum advantage of membership inference attacks (see \cref{sec:proposed-framework}). Second, it shows maximum difference between any singling out, attribute inference, or data reconstruction success rate and the attack's baseline success rate prior to observing the mechanism output~\cite{cherubin2024closed,kulynych2025unifying}.

\begin{table}[tb]
    \centering
    \caption{Values of $\mu$ corresponding to common values of $(\varepsilon, \delta)$.}
    \label{tab:epsilon_mu}
    \begin{tabular}{crrr}
    $\varepsilon \downarrow / \delta \rightarrow$ & $10^{-5}$ & $10^{-6}$ & $10^{-9}$ \\
    \cmidrule(l{3pt}r{3pt}){1-1}\cmidrule(l{3pt}r{3pt}){2-4}
    0.1 & 0.03 & 0.03 & 0.02 \\
    0.5 & 0.14 & 0.12 & 0.09 \\
    1.0 & 0.27 & 0.24 & 0.18 \\
    2.0 & 0.50 & 0.45 & 0.35 \\
    4.0 & 0.92 & 0.84 & 0.67 \\
    6.0 & 1.31 & 1.20 & 0.97 \\
    8.0 & 1.67 & 1.53 & 1.26 \\
    10.0 & 2.00 & 1.85 & 1.54 \\
    \end{tabular}
    \vspace{-1em}
\end{table}

\begin{table*}[]
    \centering
        \caption{Mapping of $\mu$ values to operational risk. We show maximum TPR @ FPR of membership inference attacks, which can also be interpreted as maximum success rate at given baseline success rate of attacks such as singling out, attribute inference, data reconstruction attacks. We also show $\eta$-TV privacy, which corresponds to maximum advantage: TPR $-$ FPR for membership inference attacks, and success rate minus the baseline rate for other attacks.}
    \label{tab:risk}

\begin{tabular}{ccrrrrrrrrr}
& & \multicolumn{9}{c}{TPR @ FPR} \\
$\mu$ & $\eta$ & 0.0001 & 0.001 & 0.01 & 0.1 & 0.25 & 0.50 & 0.75 & 0.95 & 0.99 \\
\cmidrule(l{3pt}r{3pt}){1-1}\cmidrule(l{3pt}r{3pt}){2-2}\cmidrule(l{3pt}r{3pt}){3-11}
0.10 & 0.0399 & 0.0001 & 0.0014 & 0.0130 & 0.1187 & 0.2828 & 0.5398 & 0.7807 & 0.9595 & 0.9924 \\
0.25 & 0.0995 & 0.0003 & 0.0023 & 0.0189 & 0.1511 & 0.3356 & 0.5987 & 0.8224 & 0.9709 & 0.9950 \\
0.50 & 0.1974 & 0.0006 & 0.0048 & 0.0339 & 0.2172 & 0.4307 & 0.6915 & 0.8799 & 0.9840 & 0.9976 \\
0.75 & 0.2923 & 0.0015 & 0.0096 & 0.0575 & 0.2975 & 0.5301 & 0.7734 & 0.9228 & 0.9917 & 0.9990 \\
1.00 & 0.3829 & 0.0033 & 0.0183 & 0.0924 & 0.3891 & 0.6276 & 0.8413 & 0.9530 & 0.9959 & 0.9996 \\
1.25 & 0.4680 & 0.0068 & 0.0329 & 0.1409 & 0.4874 & 0.7175 & 0.8944 & 0.9729 & 0.9981 & 0.9998 \\
1.50 & 0.5467 & 0.0132 & 0.0559 & 0.2043 & 0.5865 & 0.7955 & 0.9332 & 0.9852 & 0.9992 & 0.9999 \\
2.00 & 0.6827 & 0.0428 & 0.1378 & 0.3721 & 0.7638 & 0.9075 & 0.9772 & 0.9963 & 0.9999 & 1.0000 \\
\end{tabular}
\end{table*}

\section{Non-Uses of GDP and Open Problems}
\label{sec:non-uses}
In this section, we show examples of DP definitions and primitives that are either not tightly characterized by GDP, or whose tight characterization in terms of GDP or privacy profiles is still an open problem. If these mechanisms are encountered in practice, we suggest that the practitioners report the tightest privacy guarantee available (see \cref{sec:proposed-framework} for details). We provide the proofs of formal statements in \appref{app:proofs}.

\paragraph{Mechanisms that are only known to satisfy a single DP guarantee} The mechanisms that are only known to satisfy $\varepsilon$-DP are not well-characterized by GDP.
\begin{restatable}{proposition}{dptogdp}
    \label{stmt:dp-to-gdp}
    Any $\varepsilon$-DP mechanism satisfies GDP with $\mu = -2 \Phi^{-1}\left(\frac{1}{e^\varepsilon + 1}\right)$. 
\end{restatable}
\cref{fig:large_combined_figure} (right) shows the resulting trade-off curve using randomized response as the $\varepsilon$-DP mechanism. Although GDP tightly captures the point closest to the origin as well as $\text{FPR} \in \{0, 1\}$, it is suboptimal for other regimes. In particular, it is extremely conservative in the low \textsc{fpr} regime, and reporting the GDP guarantee has a regret of $0.058$. 

Moreover, mechanisms that are \emph{only} known to satisfy a single $(\varepsilon, \delta)$-DP guarantee for $\delta > 0$ do not provide any meaningful GDP guarantee.
\begin{restatable}{proposition}{adptogdp}
    \label{stmt:adp-to-gdp}
    For any $\varepsilon \in [0, \infty), \delta \in (0, 1]$, there exists an $(\varepsilon, \delta)$-DP mechanism that does not satisfy $\mu$-GDP for any finite $\mu$.
\end{restatable}
This is a problem particularly for mechanisms that can catastrophically fail, i.e., their trade-off curve is such that $f(0) < 1$, e.g., leaky randomized response mechanism. %
In such cases, GDP is not applicable.

\paragraph{Report-Noisy-Max Mechanisms}
The exponential mechanism
is a special case of the Report-Noisy-Max (RNM) mechanism~\cite{dwork2014algorithmic}. This mechanism is used, e.g., in the PATE DP learning framework~\cite{papernot2018scalable,papernot2017semi}. It remains an open problem if a closed form GDP, trade-off function, or privacy profile characterization exists for the general RNM mechanism.

\paragraph{Smooth Sensitivity and Propose-Test-Release} Frameworks such as smooth sensitivity~\cite{nissim2007smooth} and Propose-Test-Release (PTR)~\cite{dwork2006calibrating} only have known analyses in terms of pure or approximate DP. Obtaining an analysis in terms of trade-off curves or privacy profiles for these mechanisms or their variants, is an open question. 

\section{Concluding Remarks}

In this paper, we used recent advances in DP to derive a correct, i.e., pessimistic, Gaussian DP guarantee for any mechanism which admits tight analyses in terms of privacy profiles or trade-off curves, such as DP-SGD. We empirically showed that, in many practical scenarios, GDP---a concise, single-parameter representation of privacy guarantees---carries practically equivalent information about the privacy guarantees of an algorithm as the entire privacy profile, unlike other parameterizations such as a single $(\varepsilon, \delta)$-DP pair. 

These theoretical and empirical findings have important practical implications when reporting privacy guarantees, as there are at least two distinct audiences to consider: i) regulators or others defining allowable privacy budget; ii) researchers and engineers developing and comparing algorithms. Reporting $\mu$-GDP is particularly well-suited for the first group, as it provides a compact representation of the full privacy profile that is common for many practical mechanisms, from which, e.g., one can derive any required interpretable notion of privacy risk. For researchers and developers, $\mu$-GDP offers significant advantages in comparing mechanisms across different settings, though these users will sometimes need detailed analysis of mechanisms whose privacy properties are inaccurately captured by GDP.
The required information is contained in trade-off curves or privacy profiles, and reporting them numerically would be one possible approach.

For many common ML applications, our proposed framework enables concise communication of privacy guarantees with a single number, correct comparability of mechanisms across different settings, and precise characterizations of risks.

\section{Alternative Viewpoints}

We have encountered two common reactions to the proposal: (1) GDP is an asymptotic notion of privacy (which is incorrect), and that (2) we propose another two-parameter notion of privacy like $(\varepsilon, \delta)$-DP. We address these two points below.

\paragraph{A misconception that GDP is an asymptotic guarantee} Earlier work on GDP has focused on deriving asymptotic approximations of GDP~\citep{dong2022gaussian,Bu2020Deep}, and these approaches can lead to optimistic results (i.e. underestimating $\varepsilon$ instead of overestimating) ~\citep{gopi2021numerical}. Because of the focus in early work on asymptotic analyses, there is a common misconception that GDP is an asymptotic guarantee in principle, which is not true.
Our proposal uses pessimistic, non-asymptotic GDP bounds, which can be easily computed from standard numerical privacy accountants, as we described in \Cref{sec:proposed-framework}.

\paragraph{Difference in the semantics of regret and $\delta$} Our proposal suggests to optionally check whether a GDP guarantee characterizes the true trade-off curve with a low enough representation regret. Thus, one might argue that this is effectively a two-parameter characterization of privacy $(\mu, \Delta)$ just like $(\varepsilon, \delta)$, where $\Delta$ is the regret value. There is a crucial difference to $(\varepsilon, \delta)$, however. As we propose to find $\mu$ that provides a pessimistic bound on the true trade-off curve, \emph{regardless of regret,} the $\mu$ values are directly comparable across any mechanisms, datasets, papers, deployments, or settings. This is in stark contrast to $(\varepsilon, \delta)$, in which the values of $\varepsilon$ are only directly comparable if $\delta$ is the same. As there is no one standard value of $\delta$, and $\delta$ is normally data-dependent, this is unlikely. At the same time, if regret is small enough, e.g., $10^{-2}$, for practical purposes, it may be ignored.

\ifthenelse{\boolean{preprint}}{

}{
\section{LLM Usage Considerations}

\textbf{Originality.}
In this manuscript, LLMs were used either for generating boilerplate LaTeX code, or for editorial purposes. All outputs were inspected by the authors to ensure accuracy and originality. Our motivation for using an LLM was limited to improving readability (grammar, concision, and style). The literature review (identifying, reading, and selecting prior work) and all citations are the authors’ own; we did not rely on an LLM to discover or summarize related work beyond superficial wording suggestions. 

\textbf{Transparency.}
LLMs were not integral to the paper’s methodology. To avoid LLM-induced errors and non-determinism, we constrained prompts to local rewrites of text we had already written and manually reviewed every suggested change against our intended meaning and the cited sources. No technical ideas, figures, claims, definitions, theorems, or proofs were generated by an LLM. 

\textbf{Responsibility.}
We did not train any ML models for this work; results about trained models come from cited prior work. Our code (see \cref{sec:intro}) runs in milliseconds on a single-core CPU and therefore has negligible environmental impact.
}

\section*{Acknowledgements}
This project is supported by the U.S. Department of Energy, Office of Science, Office of Advanced Scientific Computing Research, Department of Energy Computational Science Graduate Fellowship under Award Number DE-SC0022158, by Swiss National Science Foundation under Award Numbers 10003518 and 237378, by the Research Council of Finland (Flagship programme: Finnish Center for Artificial Intelligence, FCAI, Grant 356499 and Grant 359111), the Strategic Research Council at the Research Council of Finland (Grant 358247), and the European Union (Project 101070617), and is part of the SYNTHIA project. SYNTHIA (Synthetic Data Generation framework for integrated validation of use cases and AI healthcare applications) is supported by the Innovative Health Initiative Joint Undertaking (IHI JU) under grant agreement No. 101172872. Thus, the project is partially funded by the European Union, the private members, and those contributing partners of the IHI JU. Views and opinions expressed are however those of the authors only and do not necessarily reflect those of the aforementioned parties. Neither of the aforementioned parties can be held responsible for them. AH acknowledges the research environment provided by ELLIS Institute Finland.

\bibliography{main}
\bibliographystyle{plainnat}

\clearpage
\ifthenelse{\boolean{preprint}}{
\appendix
}{
\appendices
}

\section{Detailed Background on Privacy Representations}\label{app:detailed}

In this section, we detail the following: the hockeystick divergence based privacy definitions of pure and approximate DP in \appref{app:adp}, the R\'enyi divergence based R\'enyi DP and zero-concentrated DP in \appref{app:rdp}, numerical accountants in \appref{app:profiles}, the hypothesis testing based definitions of $f$-DP  along with its connections to numerical accountants in \appref{app:fdp}, $\mu$-GDP in \appref{app:gdp}, and the optimal conversions from various privacy guarantees to $f$-DP in \appref{app:trade-off-curves-repr}. We begin with an overview of notation. 

\paragraph{Notation} A randomized algorithm (or mechanism) $M$ maps input datasets to probability distributions over some output space. Let $S \in \sD^N$ denote a dataset with $N$ individuals over a data record space $\sD$. We use $S \simeq S'$ to denote when two datasets are neighbouring under an (arbitrary) neighbouring relation. Let $\Theta$ denote the output space, and a specific output as $\theta \in \Theta$. We use $M(S)$ to denote both the probability distribution over $\Theta$ and the underlying random variable. We use $M \circ \tilde{M}$ to denote the adaptive \emph{composition} of two mechanisms, which outputs $M(S)$ and $\tilde{M}(S)$, where $M$ can also take the output of $\tilde{M}$ as an auxiliary input. The discussion in this section will focus exclusively on adaptive composition. $\Phi$ denotes the CDF of the standard normal distribution.

\subsection{Pure and Approximate DP}\label{app:adp}

It is useful for our discussion to use a version of the standard definition of differential privacy in terms of the hockey-stick divergence. Let $(P,Q)$ denote absolutely continuous densities with respect to some measure over some domain $\mathcal{O}$:
\begin{definition}[See, e.g., \citealp{asoodehconversion}]
    The $\gamma$-Hockey-stick divergence from distribution $P$ to $Q$ is:
    \begin{equation}
        H_\gamma(P ~\|~ Q) = \sup_{E \subseteq \mathcal{O}} [ Q(E) - \gamma P(E) ],
    \end{equation}
    where $\gamma \geq 0$.
\end{definition}
\begin{definition}[\citealp{dwork2006calibrating,dwork2014algorithmic}] For  $\varepsilon \in \sR, \delta \in [0, 1)$, a mechanism $M$ satisfies $(\varepsilon, \delta)$-DP iff for all $S \simeq S'$:
\begin{equation}
    H_{e^\varepsilon}(M(S) ~\|~ M(S')) \leq \delta.
\end{equation}
\end{definition}We say that the mechanism satisfies \emph{pure DP} if $\delta = 0$ and \emph{approximate DP} otherwise. The celebrated basic composition theorem \cite{dwork2014algorithmic} says that if $M$ satisfies $(\varepsilon, \delta)$-DP and $\tilde{M}$ satisfies $(\tilde{\varepsilon}, \tilde{\delta})$-DP, then $M \circ \tilde{M}$ satisfies $(\varepsilon + \tilde{\varepsilon}, \delta+  \tilde{\delta})$-DP. Subsequent analyses showed that this result can be improved. For the composition of $T$ arbitrary $(\varepsilon, \delta)$-DP algorithms, the optimal parameters admit a closed-form \cite{kairouz2015composition}. However, the computation for general heterogeneous composition (i.e. when mechanism $M_i$ has privacy parameters $(\varepsilon_i, \delta_i)$) is \#P-complete \cite{vadhanpsharp}, and hence only approximate algorithms that compute the composition to arbitrary accuracy are feasible in practice. 

\subsection{Privacy Definitions Based on R\'enyi-Divergence}\label{app:rdp}
The lack of simple composition results for approximate DP guarantees, especially in the context of analyzing DP-SGD, is what led \citet{mironov2017renyi} to propose R\'enyi-based privacy definitions:
\begin{definition}[\citealp{mironov2017renyi}] For $t\geq1, \varepsilon(t) \geq 0$, a mechanism $M(\cdot)$ satisfies $(t, \varepsilon(t))$-RDP iff for all $S \simeq S'$:
\begin{equation}
    D_{t}(M(S) ~\|~ M(S'))\leq \varepsilon(t).
\end{equation}
\end{definition}
Concentrated DP~\cite{dwork2016concentrated,bun2016concentrated,bun2018composable} is a related family of privacy definitions. In this discussion, we focus on the notion of $\rho$-zCDP due to \citet{bun2016concentrated}. A mechanism satisfies $\rho$-zCDP if it satisfies $(t,\rho \: t)$-RDP for all $t \geq1 $ and some $\rho \geq 0$. \emph{Optimal} compositions for these two privacy notions are similar to the basic composition of approximate DP: if $M$ satisfies $(t, \rho(t))$-RDP  and $\tilde{M}$ satisfies $(t, \tilde{\rho}(t))$-RDP, then $M \circ \tilde{M}$ satisfies $(t, \rho(t) + \tilde{\rho}(t))$-RDP. The result for $\rho$-zCDP follows as a special case. These simple composition rules contrast the involved computations for the optimal composition results in approximate DP.

A single RDP guarantee implies a continuum of approximate DP guarantees, with the optimal conversion given by \citet{asoodehconversion}. This means that R\'enyi-based approaches provide a more precise model of the privacy guarantees for any fixed mechanism compared to approximate DP. Consequently, these approaches enable a straightforward workflow for composition: first, compose R\'enyi guarantees, and then convert them to approximate DP guarantees as the final step. This workflow yielded significantly tighter approximate DP guarantees \cite{abadi2016deep}, and as such the researchers achieved their initial goal of fixing the perceived shortcomings of approximate DP. It was shown in later work, however,that the conversion from R\'enyi divergences to approximate DP is always lossy~\cite{balle2020hypothesis, asoodehconversion}, hence tighter bounds on approximate DP are possible with more advanced numerical approaches that compose approximate DP guarantees.

\subsection{Accountants, Privacy Profiles, and Dominating Pairs}\label{app:profiles}
A different line of work \cite{Meiserfft, koskela20b, koskela2021computingdifferentialprivacyguarantees, koskela21a, gopi2021numerical, doroshenko2022connectdotstighterdiscrete} focused on improving numerical algorithms that computed the approximate DP guarantees under composition without using R\'enyi divergences. In particular, these approaches focused on the heterogenous case where one aims to compose mechanisms $M_i, i \in [T]$, where each mechanism $M_i$ satisfies a collection of DP guarantees $\{\varepsilon_{i,j}, \delta_{i,j}\}_{j = 1}^k$. This is a strict generalization of the case explored by \citet{vadhanpsharp}. In its most general form, each mechanism $M_i$ satisfies a \emph{continuum} of privacy guarantees, which we refer to as the privacy profile function:
\begin{definition}[\citealp{balle2018privacy}]
    A mechanism $M(\cdot)$ has a privacy profile $\delta(\varepsilon)$ if for every $\varepsilon \in \mathbb{R}$, it is $(\varepsilon, \delta(\varepsilon))$-DP. 
\end{definition}
Hence, the goal of this line of work was to assume that mechanism $M_i$ has a privacy profile $\delta_i(\varepsilon)$, and the goal is to find the privacy profile of $M_1 \circ M_2 \circ \ldots \circ M_T$. The negative result from \citet{vadhanpsharp} implies that these privacy profiles are intractable to compute. Therefore, numerical algorithms called accountants are used to compute tight upper bounds to these privacy profiles. Many accountants, including the current state-of-the art~\cite{doroshenko2022connectdotstighterdiscrete}, makes use of a notion of dominating pairs, which we review below:
\begin{definition}[\citealp{zhu2022optimalaccountingdifferentialprivacy}]\label{def:domp}
    A pair of distributions $(P,Q)$  are a dominating pair to a pair of distributions $(A,B)$, denoted by $(A,B) \preceq (P,Q)$ if, for all $\gamma \ge 0$ we have:
    \begin{equation}
         H_{\gamma}(A ~\|~ B) \leq H_{\gamma}(P ~\|~ Q).
    \end{equation}
    Moreover, a pair of distributions $(P,Q)$ dominates a mechanism $M$ if $(M(S), M(S')) \preceq (P,Q)$ for all $S \simeq S'$. We denote this by $M \preceq (P,Q)$.
    If equality holds for all $\gamma$ in \cref{def:domp}, then we say $(P,Q)$ are tightly dominating.
\end{definition}
With dominating pairs, it is possible to compute privacy profiles for mechanisms under composition:
\begin{theorem}
    If $M \preceq (P,Q)$ and $\tilde{M} \preceq (\tilde{P}, \tilde{Q})$, then $M \circ \tilde{M} \preceq (P \otimes \tilde{P}, Q \otimes \tilde{Q})$, where $P \otimes \tilde{P}$ denotes the product distribution of $P$ and $\tilde{P}$.
\end{theorem}
To convert $H_{e^\varepsilon}(P \otimes \tilde{P} ~\|~ Q \otimes \tilde{Q})$ 
into an efficiently computable form, we consider a notion of privacy loss random variables (PLRVs)~\cite{dwork2016concentrated}. Let $[x]^+ = \max(0,x)$. 
\begin{theorem}[Based on \citet{gopi2021numerical}]
    \label{thm:plrv}
    Given a dominating pair $(P,Q)$ for mechanism $M$, define PLRVs $X = \log \nicefrac{Q(o)}{P(o)}, o \sim P$, and $Y = \log \nicefrac{Q(o)}{P(o)}, o \sim Q$. The mechanism has a privacy profile:
    \begin{equation}\label{eq:app_priv_curve}
    \delta(\varepsilon) = \mathbb{E}_{y \sim Y}\left [ 1 - e^{\varepsilon - y}\right ]^+.
    \end{equation}
    Moreover, if a mechanism $\tilde{M}$ has PLRVs $\tilde{X}, \tilde{Y}$, then $M \circ \tilde{M}$ has PLRVs $X + \tilde{X}, Y + \tilde{Y}$ and privacy profile:
    \begin{equation}
    \delta(\varepsilon) = \mathbb{E}_{y \sim Y + \tilde{Y}}\left [ 1 - e^{\varepsilon - y}\right ]^+.
    \end{equation}
\end{theorem}
In other words, PLRVs turn compositions of mechanisms into convolutions of random variables. 
In particular, it is possible to choose $(P, Q)$ in such a way that the composition can be efficiently computed with fast Fourier transform (FTT)~\citep{koskela20b}. This approach yields significantly more precise approximate DP guarantees than R\'enyi-based workflows. 

In summary, if the goal is to find the privacy profile of $M_1 \circ M_2 \circ \ldots \circ M_T$ given that mechanisms $M_i$ have privacy profiles $\delta_i(\varepsilon)$, then the workflow is: (1) compute dominating pairs $(P_i, Q_i)$ to mechanism $M_i$ (we would recommend using the Connect-The-Dots approach of \citet{doroshenko2022connectdotstighterdiscrete}, as it is optimal), (2) compute the PLRVs $(X_i,Y_i)$ for mechanism $M_i$ using \cref{thm:plrv}, (3) Compute the PLRV $Y_T$ of the composed mechanism via $Y_T = \sum_i Y_i$ using the FFT, (4) use \cref{eq:app_priv_curve} to compute the privacy profile of $M_1 \circ M_2 \circ \ldots \circ M_T$. These algorithms compute profiles to accuracy nearly matching the lower bound of a privacy audit where the adversary is free to choose the entire (often pathological) training dataset \cite{nasr2021adversary,nasr2023tight}. Given these results, we treat the analyses of numerical accountants as precise.

\subsection{Hypothesis Testing Interpretation of DP and Numerical Accountants}\label{app:fdp}
A independent line of work reformulated differential privacy in terms of hypothesis tests. Though this connection was pointed out early in DP's history~\cite{wasserman2010statistical}, its full implications were explored much later in \cite{kairouz2015composition, dong2022gaussian}. More important to the discussion in this work, it turns out that privacy profiles as defined in \appref{app:profiles} are closely related to $f$-DP a defined in \cite{dong2022gaussian}. In this section, we define $f$-DP then connect it to privacy profiles. Then, we show that the numerical accountants discussed in \appref{app:profiles} can be used to compute trade-off curves too. We conclude with introducing $\mu$-GDP.

Consider a binary hypothesis test where an adversary observes an outcome $o \in \mathcal{O}$ and their goal is to determine if $o$ came from distribution $P$ or $Q$. This test is completely characterized by the \textit{trade-off function} $T(P,Q): \alpha \rightarrow \beta(\alpha)$, where $(\alpha, \beta(\alpha))$ denote the Type-I/II errors of the most powerful level $\alpha$ test between $P$ and $Q$ with null hypothesis $H_0: o \sim P$  and alternative $H_1: o \sim Q$. Note that $T(P, Q)$ is convex, continuous, non-increasing, and for all $\alpha \in [0, 1]$, $T(P, Q)(\alpha) \leq 1 - \alpha$.

\citet{dong2022gaussian} use these trade-off functions to propose $f$-DP. It turns out that the dominating pairs from \cref{def:domp} are a natural choice to define $f$-DP:
\begin{definition}\label{def:fdp_for_metric}
    A mechanism $M$ is $f$-DP iff there exists $(P,Q)$ where $f = T(P,Q)$ and $M \preceq (P,Q)$. 
\end{definition}
\citet{zhu2022optimalaccountingdifferentialprivacy} showed it is possible to compute a tightly dominating pair $(P^*, Q^*)$ for any mechanism $M$. Thus, any mechanism $M$ has an \textit{associated trade-off curve} $f_M = T(P^*, Q^*)$. That a mechanism satisfies $f$-DP means that $f(\alpha) \leq f_M(\alpha)$ for all $\alpha \in [0, 1]$, and that there exists a pair $(P,Q)$ such that $f(\alpha) = T(P, Q)(\alpha)$~\cite{kulynych2024attack}.
An $f$-DP guarantee is equivalent to a privacy profile: 
\begin{theorem}[\citealp{dong2022gaussian}]
    A mechanism is $f$-DP if and only if it satisfies $(\varepsilon, 1 + f^*(-e^\varepsilon))$-DP for all $\varepsilon\in \mathbb{R}$\footnote{If $f$ is symmetric, only $\varepsilon \geq 0$ is needed.}, where $f^*$ denotes the convex conjugate of $f$. 
\end{theorem}
Note that all the previously discussed privacy definitions---$(\varepsilon, \delta)$-DP, R\'enyi DP, zCDP---imply both a trade-off curve and a privacy profile, which we detail in \appref{app:delta-divergence}. 

The numerical accountants from \appref{app:profiles} can be used to compute trade-off functions under composition:
\begin{theorem}[\citealp{kulynych2024attack}]\label{thm:fdp}
    Let $(P,Q)$ be a dominating pair for a mechanism $M$ and $(X,Y)$ be the associated PLRVs as defined in \cref{thm:plrv}. Suppose the PLRVs share the same finite support $\Omega = \{\omega_0, \ldots, \omega_k\}$. Then, $T(P,Q)$ is piecewise linear with breakpoints $\{ \Pr[X > \omega_i], \Pr[Y \leq \omega_i] \}_{i = 0}^k$. 
\end{theorem}

We copy \cref{alg:get-beta} from \citet{kulynych2024attack} for completeness. This algorithm simply implements the steps outlined in \cref{thm:fdp}. We remark that the PLRVs $(X, Y)$ can always be chosen to have the same finite support, and \citet{doroshenko2022connectdotstighterdiscrete} provided the optimal algorithm for how to construct these PLRVs. 

\begin{algorithm}[t]
\caption{Compute $f_{(X, Y)}(\alpha)$ for discrete privacy loss random variables $(X, Y)$~\cite{kulynych2024attack}}
\label{alg:get-beta}
\begin{algorithmic}[1]
\REQUIRE PMF $\Pr[X = x_i]$ over grid $\{x_1, x_2, \ldots, x_k\}$ with $x_1 < x_2 < \ldots < x_k$
\REQUIRE PMF $\Pr[Y = y_j]$ over grid $\{y_1, y_2, \ldots, y_l\}$ with $y_1 < y_2 < \ldots < y_l$
\STATE \(t \leftarrow \min \{i \in \{0, 1, \ldots, k\} \mid \Pr[X > x_{i}] \leq \alpha \}, \text{ where } x_0 \define -\infty \)
\STATE \(\gamma \leftarrow \frac{\alpha - \Pr[X > x_{t}]}{\Pr[X = x_{t}]}\)
\STATE $f(\alpha) \leftarrow \Pr[Y \leq x_{t}] - \gamma \Pr[Y = x_{t}]$
\end{algorithmic}
\end{algorithm}

\paragraph{From $f$-DP to Operational Privacy Risk}
Assuming that the neighbouring relation $S \simeq S'$ is such that the datasets differ by a single record, i.e. $S' = \{S \cup z\}$ for some $z$, the hypothesis testing setup described previously can also be seen as a membership inference attack (MIA)~\cite{shokri2017membership} on the sample $z$. In this framework, the adversary aims to determine if a given output $\theta \in \Theta$ came from $M(S)$ or $M(S')$ for some neighbouring datasets $S \simeq S'$. Such an attack is equivalent to a binary hypothesis test~\cite{wasserman2010statistical, kairouz2015composition, dong2022gaussian}:
\begin{align}
H_0: \theta \sim M(S), \quad H_1: \theta \sim M(S'),
\end{align}
where the MIA is modelled as a test $\phi: \Theta \rightarrow [0,1]$ which associates a given output $\theta$ to the probability of the null hypothesis $H_0$ being rejected. We can analyze this hypothesis test through the trade-off between the attainable \emph{false positive rate} (FPR) $\alpha_\phi \define \smash{\E_{\theta \sim M(S)}}[\phi(\theta)]$ and \emph{false negative rate} (FNR) $\beta_\phi \define 1 - \smash{\E_{\theta \sim M(S')}}[\phi(\theta)]$. This trade-off function is the same as defined before, except here we have the extra intuition that the goal of the adversary is to identify one particular member $z$ in the dataset. We note that a function $f: [0,1] \rightarrow [0,1]$ is a trade-off function iff $f$ is convex, continuous, non-increasing, and $f(x) \leq 1 - x$ for $x \in [0,1]$. We denote the set of functions with these properties by $\mathcal{F}$. We can now state the more standard $f$-DP definition:

\begin{definition}[\citealp{dong2022gaussian}]\label{def:def_of_fdp}
A mechanism $M$ satisfies $f$-DP, where $f \in \mathcal{F}$, if for all $\alpha \in [0,1]$, we have
$\inf_{S \simeq S'} T(M(S), M(S'))(\alpha) \geq f(\alpha)$.
\end{definition}
This is the standard definition of $f$-DP, though we presented it earlier using dominating pairs to make the connection to numerical accountants clear. However, the standard approach makes it clear that $\beta = f(\alpha)$ can be interpreted as FNR of the worst-case strong-adversary membership inference attack with FPR $\alpha$~\cite{nasr2021adversary}. Moreover, it also tightly bounds other notions of attack risk such as maximum accuracy of attribute inference or reconstruction attacks~\cite{kaissis2023bounding,hayes2024bounding,kulynych2025unifying}.

\subsection{Gaussian Differential Privacy}\label{app:gdp}
Gaussian Differential Privacy (GDP) is a special case of $f$-DP which conveniently characterizes common private mechanisms based on the Gaussian mechanism:
\begin{definition}[\citealp{dong2022gaussian}]
    A mechanism $M(\cdot)$ satisfies $\mu$-GDP iff it is $f_\mu$-DP with:
    \begin{equation}\label{eq:gdp-app}
        f_\mu(\alpha) = \Phi(\Phi^{-1}(1- \alpha) - \mu),
    \end{equation}
    where $\Phi$ denotes the CDF and $\Phi^{-1}$ the quantile function of the standard normal distribution. 
\end{definition}

The introduction of $\mu$-GDP due to \citet{dong2022gaussian} received a mixed response from the community. One of its key observations was that any privacy definition framed through a hypothesis testing approach to ``indistinguishability'' will, under composition, converge to the guarantees of Gaussian Differential Privacy (GDP). The proof of this convergence established a uniform convergence in the trade-off function to a Gaussian trade-off function in the limit as the number of compositions went to infinity (see, e.g. Theorem 5.2 in \cite{dong2022gaussian}). Most importantly, it was unclear whether this $\mu$-GDP asymptotic lower-bounded the trade-off function (in which case, the $\mu$-GDP asymptotic yielded a valid $f$-DP guarantee) or upper-bounded the trade-off function (in which case the asymptotic is not valid privacy guarantee), when applied to algorithms with finite compositions.

A notable follow-up study applied $\mu$-GDP to the analysis of DP-SGD~\cite{Bu2020Deep}, and derived an asymptotic closed-form expression for $\mu$ in the limit as compositions tends to infinity. Unfortunately, this expression was shown to lower-bound the privacy profile by \citet{gopi2021numerical}, which equivalently meant that the asymptotic $\mu$-GDP trade-off function upper-bounded the true underlying trade-off function. The core issue here lies in the fact that although privacy amplification through Poisson subsampling can be tightly captured by $f$-DP, the resulting trade-off curve deviates from a Gaussian form. This deviation complicates the theoretical analysis of $\mu$-GDP for subsampled mechanisms.

Given these challenges, it may seem surprising that we advocate for $\mu$-GDP in machine learning applications. However, our approach addresses two critical differences that set it apart from the previous approaches:

\begin{enumerate}
    \item \textbf{No reliance on asymptotic expressions:} Rather than using approximations for $\mu$, we compute the trade-off curve \emph{numerically} using existing accountants. We then perform a post-hoc optimization to find the tightest $\mu$ ensuring the mechanism adheres to $\mu$-GDP. 
    \item \textbf{Freedom from distributional assumptions:} As our method avoids asymptotic approximations, it does not require any specific assumptions about the moments or the underlying distribution of the privacy loss random variable, which are of central importance in asymptotic approximations.
    \item \textbf{Correctness guarantee:} Our method ensures that the obtained $\mu$-GDP guarantee is pessimistic, i.e., does not overestimate the privacy protection.
\end{enumerate}

\subsection{Associated Trade-off Curves}
\label{app:delta-divergence}
\label{app:trade-off-curves-repr}

Each of the privacy definitions discussed before has an associated trade-off curve, which we provide for reference next.

\paragraph{ADP} If a mechanism satisfies $(\varepsilon, \delta)$-DP, it satisfies $f_{(\varepsilon, \delta)}$-DP~\cite{dong2022gaussian}:
    \[
        f_{\varepsilon, \delta}(x) = \max\{0, 1 - e^\varepsilon x - \delta, e^{-\varepsilon} (1 - x - \delta)\}.
    \]

\paragraph{GDP} If a mechanism satisfies $\mu$-GDP, then by definition it satisfies $f_\mu$-DP where 
    \[
    f_\mu(\alpha) = \Phi( \Phi^{-1}(1 - \alpha) - \mu).
    \]

\paragraph{RDP} If a mechanism satisfies $(t, \varepsilon)$-RDP, it satisfies $f_{(t, \varepsilon)}$-DP, where $\beta = f_{(t, \varepsilon)}(\alpha)$ is defined by the following inequalities~\cite{balle2020hypothesis,asoodehconversion,zhu2022optimalaccountingdifferentialprivacy}, for $t > 1$:
    \[
        \begin{aligned}
        (1 - \beta)^t \alpha^{1 - t} + \beta^t (1 - \alpha)^{1 - t} \leq  e^{(t - 1) \varepsilon} \\
        (1 - \alpha)^t \beta^{1 - t} + \alpha^t (1 - \beta)^{1 - t} \leq e^{(t - 1) \varepsilon},
        \end{aligned}
    \]
    for $t \in [\nicefrac{1}{2}, 1)$:
    \[
        \begin{aligned}
        (1 - \beta)^t \alpha^{1 - t} + \beta^t (1 - \alpha)^{1 - t} \geq  e^{(t - 1) \varepsilon} \\
        (1 - \alpha)^t \beta^{1 - t} + \alpha^t (1 - \beta)^{1 - t} \geq e^{(t - 1) \varepsilon},
        \end{aligned}
    \]
    and for $t = 1$:
    \[
        \begin{aligned}
        \alpha \log\left(\frac{\alpha}{1 - \beta}\right) + (1 - \alpha) \log\left(\frac{1-\alpha}{\beta}\right) \leq \varepsilon \\
        \beta \log\left(\frac{\beta}{1 - \alpha}\right) + (1 - \beta) \log\left(\frac{1-\beta}{\alpha}\right) \leq \varepsilon \\
        \end{aligned}
    \]

If a mechanism satisfies a continuum of $(t, \varepsilon(t))$-RDP guarantees, then the trade-off function $\beta = f_{(t, \varepsilon(t))}(\alpha)$ can be obtained by running the above for fixed alpha over the collection of $(t, \varepsilon(t))$-RDP guarantees, then taking the minimum over the resulting $\beta$. 

\paragraph{zCDP} If a mechanism satisfies $\rho$-zCDP, we can set $\varepsilon(t) = \rho t$ and use the previous result for a continuum of RDP guarentees to get the trade-off function for zCDP as a special case. No known closed-form expressions for this trade-off function are known.  

\section{MIA Success Bounds against a GDP mechanism}\label{app:gaussian_app}

In this section, we provide further intuition for the connection between privacy profiles and trade-off functions. In \cref{fig:mia-bounds-app}, the top figure shows the profile $(\varepsilon, \delta(\varepsilon))$ profile of a DP algorithm calibrated for $\varepsilon=8, \delta=10^{-5}$ (big blue dot).
The profile is based on Gaussian differential privacy, which accurately models the privacy of many common DP algorithms as illustrated in \Cref{sec:applicability-gdp}.
The bottom figure shows the membership inference attack~\cite{shokri2017membership} success bounds (maximum true positive rate, \text{tpr}, at fixed false positive rate, \text{fpr}) for the same DP algorithm.
The thin blue curve corresponding to bounds for $\varepsilon=8, \delta=10^{-5}$ significantly underestimates the protection.
The optimal bound (thick curve) is formed as a lower envelope of curves for different $\delta \in [0, 1]$, some of which are shown as dashed lines. The points corresponding to these curves are shown as orange dots in the top plot.

\begin{figure}[t]
    \centering
    \includegraphics{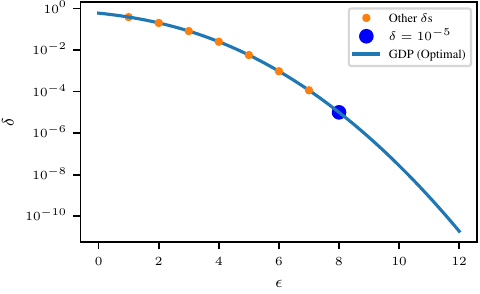}\\
    \includegraphics{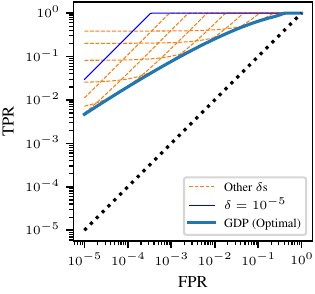}
    \caption{Top: Full privacy profile of GDP mechanism. Bottom: Corresponding membership inference attack (MIA) success bounds on a log-log ROC plot commonly used in MIA literature.}
    \label{fig:mia-bounds-app}
\end{figure}

\section{Omitted Technical Details}\label{app:tech_details}

\subsection{Details on Tight GDP Accounting}\label{app:gdp_details}
Recall that we seek to find the solution to the following problem:
\begin{equation}
    \mu^* = \inf\{\mu' \geq 0 \mid \forall \alpha: \: \: f_{\mu'}(\alpha) \leq f(\alpha)\}. \label{eq:amu}
\end{equation}

By \cref{thm:fdp} and \cref{alg:get-beta}, we can obtain the breakpoints of the piecewise linear trade-off curves using the state-of-the-art \citet{doroshenko2022connectdotstighterdiscrete} accounting. First, we show that the infinimum can be taken over just these finite breakpoints, and all other $\alpha$ can be ignored:
\begin{proposition}\label{thm:afastmu}
    Given a piecewise-linear trade-off curve $f$ with breakpoints $\{\alpha_i\}_{i=1}^k$ from \cref{thm:fdp}, we have:
    \begin{equation}
        \mu^* = \inf\{\mu' \geq 0 \mid \forall i \in [k]:\: \: f_{\mu'}(\alpha_i) \leq \beta_i\}. \label{eq:afastmu}
    \end{equation}
    \proof Follows from the convexity and monotonicity of $f$ and $f_\mu$ for all $\mu \geq 0$. \qed  
\end{proposition}
Moreover, \cref{eq:afastmu} can be easily inverted using the formula for $f_\mu$ to solve for $\mu$. We have that \cref{eq:afastmu} is equivalent to 
\begin{equation}
        \mu^* = \inf\{\mu' \geq 0 \mid \forall i \in [k]:\: \: \mu' \geq \Phi^{-1}(1-\alpha_i) -\Phi^{-1}(\beta_i)\}. \label{eq:afastermu}
    \end{equation}
which has the solution:
\begin{equation}
    \mu^* = \max_{i\in [k]} \{\Phi^{-1}(1-\alpha_i) -\Phi^{-1}(\beta_i) \}. \label{eq:amufinal}
\end{equation}

To quantify the goodness-of-fit of the $\mu^*$-GDP guarantee, we employ the symmetrized metric in \cref{sec:metric}, namely $\Delta^{\leftrightarrow}(f, f_{\mu^*})$. First, we observe that the symmetrization is not necessary in our scenario, as $f_{\mu^*}(\alpha) \leq f(\alpha)$ for $\alpha \in [0,1]$. 
\begin{restatable}{proposition}{symmequality}\label{prop:symm}
    Given a trade-off curve $f$ and $\mu^*$ obtained via \cref{eq:mu}, we have:
    \begin{equation}
        \Delta^{\leftrightarrow}(f, f_{\mu^*}) = \Delta(f, f_{\mu^*}).
    \end{equation}
\end{restatable}
Moreover, we can compute it as follows:
\begin{equation}\label{eq:symm_final}
    \begin{aligned}
        & \Delta^{\leftrightarrow}(f, f_{\mu^*}) = \Delta(f, f_{\mu^*}) =\\
        & = \inf \{ \kappa \geq 0 ~| \: \:\forall \alpha \in [0,1] : f(\alpha + \kappa) - \kappa \leq f_{\mu^*}(\alpha)\}.
    \end{aligned}
\end{equation}
\begin{proof}
As $f_{\mu^*}(\alpha) \leq f(\alpha)$ for $\alpha \in [0,1]$, it follows that $\Delta( f_{\mu^*}, f) = 0$.
\end{proof}
Moreover, this result holds for any pessimistic DP bound. The optimization problem in \cref{eq:symm_final} can be numerically solved using binary search as follows. We begin with a simple observation: at $\kappa = 1$, the inequality is trivially true since $f$ is bounded between 0 and 1, so the LHS of the inequality is always negative and the RHS is always positive. At $\kappa =0$, the inequality is not true by assumption. It follows that the set $\{ \kappa \geq 0 | \: \:\forall \ : \alpha \in [0,1]: f(\alpha + \kappa) - \kappa \leq f_{\mu}(\alpha)\}$ has the form $\{\kappa: \kappa_\text{min} \leq \kappa\}$ for some value $\kappa_\text{min}$. The most straightforward way to solve for $\kappa_\text{min}$ is via a simple binary search over $\kappa \in [0,1]$ over a sufficiently dense grid for $\alpha$. In each iteration, we narrow down the interval $[\kappa_L, \kappa_R]$ up to a prespecified tolerance $\mathsf{tol}$. Once the desired tolerance is achieved, we return $\kappa_L$ to guarantee that we underestimate $\Delta$.

\subsection{Details on Tight RDP Accounting}\label{app:rdp_details}
Let $\tilde f$ denote a pessimistic lower bound to the trade-off function of some underlying mechanism $f$. In \cref{prop:symm}, we showed that the metric $\Delta^\leftrightarrow(f,\tilde f)$, which denotes the regret in choosing to use the pessimistic lower bound $\tilde f$ over the trade-off function $f$, can be computed as: $$\Delta^\leftrightarrow(f,\tilde f)=\inf \{ \kappa \geq 0 ~| \: \:\forall \alpha \in [0,1] : f(\alpha + \kappa) - \kappa \leq f_{\mu^*}(\alpha)\}.$$ 
This form is useful if we have a trade-off function for the underlying mechanism and for the pessimistic DP bound. In the case of bounds for R\'enyi DP, the optimal trade-off functions are known and are detailed in \appref{app:trade-off-curves-repr}. In practice, however, we found these expressions to be both numerically unstable and very time-consuming to work with. The idea behind this position paper is to point out that there are numerically stable and quick ways to determine how tight a given bound is to a fixed mechanism. We found the $f$-DP bounds in \appref{app:trade-off-curves-repr} to run counter to this message, as computing $\Delta^\leftrightarrow(f,\tilde f)$ once requires solving possibly hundreds of convex optimization problems. We circumvent this problem by pointing out that the metric $\Delta^\leftrightarrow(f,\tilde f)$ can also be expressed as a function between privacy profiles. 

\begin{definition}[\citealp{kaissis2024beyond}]\label{def:metric_pc}
The metric in \cref{def:metric} can also be expressed as a function between two privacy profiles. Given two mechanisms $M , \tilde M$ with privacy profiles  $\delta(\varepsilon), \tilde{\delta}(\varepsilon)$, the $\Delta$-divergence from $M$ to $\tilde M$ is:
\begin{equation}\label{eq:delta-div-pc}
    \Delta(\delta, \tilde{\delta}) \define \inf \{ \kappa \geq 0 ~\mid~ \forall \varepsilon: \:\delta(\varepsilon) + \kappa \cdot (1 + e^{\varepsilon})  \geq \tilde{\delta}(\varepsilon)\}.
\end{equation}
\end{definition}
In the context of RDP, this expression is much more convenient to work with, as the privacy profile implied by an RDP guarantee has been the subject of many previous works \cite{mironov2017renyi, mironov2019renyidifferentialprivacysampled, canonne2020discrete, asoodehconversion, balle2020hypothesis}. While the optimal conversion from a RDP guarantee to a privacy profile is known \cite{asoodehconversion}, this conversion requires solving a convex optimization problem, and there are closed-form upper-bounds that are considerably cheaper to compute \cite{canonne2020discrete} and reasonably close to optimal in the regimes of interest. 

\cref{def:metric_pc} is hence how we computed the regret in \cref{fig:metric_table}, as it allowed us to take advantage of this rich literature. The privacy curve for RDP was calculated using the open-source \texttt{dp\_accounting} library, in particular their RDP implementation in Python. 

Going into more detail, to compute the privacy curve implied by a single $(t, \varepsilon)$-RDP pair, the conversion due to \citet{canonne2020discrete}, Proposition 12 in v4 was used. To obtain the privacy curve implied by a continuum of RDP guarantees $(t, \varepsilon(t))$, we computed a grid of $(t, \varepsilon(t))$ guarantees over a grid of $t$, computed the privacy curve for each pair, and took the minimum across all privacy curves. 

\subsection{Details on Tight zCDP Accounting for DP-SGD}\label{app:zcdp_details}

Given that we took advantage of high precision numerical accountants to compute non-asymtotic $\mu$-GDP bounds for DP-SGD, it is only fair to benchmark against $\rho$-zCDP when an equal amount of numerics are applied. In the context of DP-SGD, the exact R\'enyi divergence $\varepsilon(t)$ can be computed to arbitrary precision using the technique due to \citet{mironov2019renyidifferentialprivacysampled}. Given the R\'enyi divergence, it is straightforward to compute a tight $\rho$-zCDP guarantee in a manner very similar to how we computed a tight $\mu$-GDP guarantee from a trade-off function in \appref{app:gdp_details}. In particular: we seek to find the solution to the following problem:
\begin{equation}
    \rho^* = \inf\{\rho \geq 0 \mid \forall t > 1: \: \: \varepsilon(t) \leq t \cdot \rho\}. \label{eq:arho}
\end{equation}
Unlike \cref{thm:afastmu}, there is no additional structure to take advantage of here, but we can nevertheless numerically solve \cref{eq:arho} by fixing a fie grid of $(t, \varepsilon(t))$ guarantees over a grid of $t$, and numerically solving for $\rho^*$ via binary search. Once we have $\rho$, we apply the technique outlined in \appref{app:rdp_details} to obtain the $\rho$-zCDP privacy curve. Note that, by construction, the regret in choosing $\rho$-zCDP must be higher than using the R\'enyi divergence function. This is indeed the case in \cref{fig:metric_table}.

\section{Choice of Metric and Why $10^{-2}$?}\label{app:metric}
In this section, we overview the metric used to quantify our goodness of fit to a $\mu$-GDP guarantee, and justify our suggestion for the metric being less than $10^{-2}$. This overview is largely based on the results by  \citet{kaissis2024beyond}, restated in the notation used throughout this work. The relevant background is in Appendices \ref{app:profiles} and \ref{app:fdp}.

From the sentences following \cref{def:fdp_for_metric}, for a fixed mechanism, we have the notion of a mechanism-specific trade-off function $f_M$, which is evaluated using a tightly dominating pair. This trade-off curve is usually numerically intractable, so a numerical lower-bound is computed via \cref{alg:get-beta} using accountants described in \appref{app:profiles}, which we denote by $f_\text{acc}$. Note that $f_\text{acc}(\alpha) \leq f_M(\alpha)$ for all $\alpha \in[0,1]$, so the mechanism is $f_\text{acc}$-DP. From the discussion at the end of \appref{app:profiles}, we have that the error in these numerical lower-bounds is negligible, and so we ignore it in this work. We henceforth treat $f_\text{acc}(\alpha) = f_M(\alpha)$ for all $\alpha \in[0,1]$, and refer to this function as $f$ in the remainder of this subsection. 

Using the process outlined in \cref{sec:metric} and detailed in \appref{app:tech_details}, we find the tightest possible $\mu^*$-GDP bound such that $f_{\mu^*}(\alpha) \leq f(\alpha)$ for all $\alpha \in [0,1]$. Note that the mechanism is indeed $\mu^*$-GDP. We seek a metric for quantifying how far away $f_{\mu^*}$ is from $f$. Based on \citet{kaissis2024beyond}, consider the following metric: 
\begin{equation}
    \Delta  = \inf \{ \kappa \geq 0 ~| \: \:\forall \alpha \in [0,1] : f(\alpha + \kappa) - \kappa \leq f_{\mu^*}(\alpha)\}.
\end{equation}
That is, we have that $f_{\mu^*}(\alpha) \leq f(\alpha)$ for all $\alpha \in [0,1]$, and now we now seek the smallest shift $\kappa$ so that the sign is reversed for all $\alpha \in [0,1]$. This metric turns out to have strong decision theoretic interpretations. We provide two of them below. 

Consider the same binary hypothesis test between distributions $M(S)$ and $M(S')$ as in \appref{app:fdp}.
One can consider characterizing the hypothesis test via the \emph{minimal achievable error} of a Bayesian adversary with prior probability $\pi$ of the null hypothesis $H_0: \theta \sim M(S)$ being correct. For any fixed test $\phi$, the probability of error is simply $\pi \alpha_\phi + (1-\pi) \beta_\phi$. By considering the most powerful attack and the minimal achievable error, we get that:
\begin{equation}
    R_\text{min}(\pi) = \min_{\alpha} \pi \alpha + (1 - \pi) f(\alpha).
\end{equation}
With this setup, it is straightforward to express $\Delta$:
\begin{equation}
    \Delta = \max_{\pi \in [0,1]} R_\text{min}(\pi) - R^{\mu^*}_\text{min}(\pi).
\end{equation}
That is, $\Delta$ expresses the worst-case regret of
an analyst choosing to employ a $\mu^*$-GDP mechanism instead of the original mechanism $M$,
whereby regret is expressed in terms of the adversary’s decrease in minimum Bayes error. Choosing $\Delta < 10^{-2}$ implies that the decrease in the error of any adversarial attack changes by at most a percentage point when opting to use $\mu^*$ in place of $f$.

Next, we provide a proof of another operational interpretation using the notion of advantage of attacks from \cref{sec:proposed-framework}.

\advantage*
 
To show this, we use the following lemma.
\begin{lemma}[\citealp{kaissis2024beyond}]
    \label{stmt:pp-delta-to-adv}
    Let $f_{pp}(\alpha) = 1 - \alpha$ be the trade-off curve of a mechanism which achieves perfect privacy. For any valid trade-off curve $f$, we have:
    $\Delta^\leftrightarrow(f_{pp}, f) = \frac{1}{2} \eta(f)$.
\end{lemma}

\begin{proof}
    The result follows from \cref{stmt:pp-delta-to-adv} by triangle inequality and symmetry of $\Delta^\leftrightarrow$:
    \[
    \begin{aligned}
        \eta(\tilde f) = 2 \Delta^\leftrightarrow(f_{pp}, \tilde f) &\leq 2 \Delta^\leftrightarrow(f_{pp}, f) + 2 \Delta^\leftrightarrow(f, \tilde f) \\
        &= \eta(f) + 2 \Delta^\leftrightarrow(f, f_\mu),
    \end{aligned}
    \]
    from which we have that $\eta(\tilde f) - \eta(f) \leq 2 \Delta^\leftrightarrow(f, \tilde f)$.
    Analogously, we have:
    \[
        \begin{aligned}
        \eta(f) = 2 \Delta^\leftrightarrow(f, f_{pp}) &\leq 2 \Delta^\leftrightarrow(f, \tilde f) + 2 \Delta^\leftrightarrow(\tilde f, f_{pp}) \\
        &= 2\Delta^\leftrightarrow(f, \tilde f) + \eta(\tilde f),
        \end{aligned}
    \]
    from which we have $\eta(f) - \eta(\tilde f) \leq 2\Delta^\leftrightarrow(f, \tilde f)$. Combining the two conclusions, we get the sought form.
\end{proof}
Thus, the values $\Delta < 10^{-2}$ ensure that the highest advantage of MIAs is pessimistically over-reported by at most $2$ percentage points. 

Additionally, we present empirical results in \appref{app:plots} that show that, on both standard and log-log scales, the $\mu^*$-GDP trade-off curve closely follow the original $f$ up to numeric precision for different instantiations of DP. We emphasize that $\Delta$ is \emph{not} analogous to $\delta$ from approximate DP. Whereas $\delta$ can be interpreted as a ``failure'' probability that the privacy loss is higher than $\varepsilon$,  no such interpretation holds for $\Delta$. Indeed,  $\Delta$ quantifies how close the lower bound $\mu^*$-GDP is to $f$. There is no failure probability: $\mu^*$ is always a valid bound on $f$. However, it may be a \emph{loose} lower bound $f_{\mu^*}$ on the trade-off curve $f$ (hence a loose upper bound on privacy loss). This looseness is what is captured by $\Delta$.

\section{Additional Proofs}
\label{app:proofs}

In this section, we provide the omitted proofs of statements in the main body.
\dptogdp*

\begin{proof}
We need to find a Gaussian mechanism which dominates the randomized response mechanism $M_\varepsilon$~\cite{kairouz2015composition}.
In turn, randomized response dominates any $\varepsilon$-DP mechanism, i.e., its trade-off curve is always lower than that of any other pure DP mechanism.
The total variation of the randomized response mechanism is given by the following expression~\cite{kairouz2015composition}:
\begin{equation}
    \sup_{S \simeq S'} \mathsf{TV}(M_\varepsilon(S), M_\varepsilon(S')) = \frac{e^\varepsilon - 1}{e^\varepsilon + 1},
\end{equation}
where $\mathsf{TV}(P, Q) = H_1(P \mid Q) = \sup_{E \subseteq \Theta} P(E) - Q(E)$. For a $\mu$-GDP mechanism $M_{\mu}$, we have the following~\cite{dong2022gaussian}:
\begin{align}
    \sup_{S \simeq S'} \mathsf{TV}(M_{\mu}(S), M_{\mu}(S')) &= \Phi\left(\frac{\mu}{2}\right) - \Phi\left(-\frac{\mu}{2}\right) \\
    &= 2 \Phi\left(\frac{\mu}{2}\right) - 1
\end{align}
To ensure that $\mathsf{TV}(M_{\mu}(S), M_{\mu}(S')) = \mathsf{TV}(M_\varepsilon(S), M_\varepsilon(S'))$, it suffices to set the parameter $\mu$ as follows:
\begin{align}
    \mu^* &= 2 \Phi^{-1}\left(\frac{e^\varepsilon}{e^\varepsilon + 1}\right) \\
    &= 2 \Phi^{-1}\left(1 - \frac{1}{e^\varepsilon + 1}\right) \\
    &= -2 \Phi^{-1}\left(\frac{1}{e^\varepsilon + 1}\right)
\end{align}
Observe that both the trade-off curve of the Gaussian mechanism $T(M_{\mu^*}(S), M_{\mu^*}(S'))$ and of the randomized response $T(M_\varepsilon(S), M_\varepsilon(S'))$ pass through the points:
\[
    (1, 0), \left(\frac{1}{e^\varepsilon + 1}, \frac{1}{e^\varepsilon + 1}\right), (1, 0).
\]
As the trade-off curve of the randomized response is piecewise linear between the points above, and as the trade-off curve of the Gaussian mechanism is convex, we have that:
\begin{equation}
    T(M_\varepsilon(S), M_\varepsilon(S')) \geq
    T(M_{\mu^*}(S), M_{\mu^*}(S')).
\end{equation}
\end{proof}

\adptogdp*

\begin{proof}
We have $f_\mu(0)=1$ for any finite $\mu \geq 0$. However, $f_{\varepsilon,\delta}(0) = 1 - \delta$, hence it is impossible to choose finite $\mu$ such that $f_\mu(\alpha) \leq f_{\varepsilon,\delta}(\alpha)$ for all $\alpha \in [0, 1]$. 
\end{proof}

\section{Additional Plots}\label{app:plots}
In this section, we provide additional visualizations.

First, we show the trade-off curve plots implied by the first row of \cref{tab:de_et_al}. In particular, the privacy parameters used in the first row are given in Table 14 of \cite{de2022unlockinghighaccuracydifferentiallyprivate}. We reproduce it in \cref{app:tab_cifar}.

Second, we show the trade-off curve for each of four DP-SGD mechanisms ($\varepsilon = 1, 2, 4, 8$) in \cref{app:fig_1} as outlined in \cref{app:tab_cifar}. Since the two curves (the $\mu$-GDP trade-off curve and the original trade-off curve) are difficult to distinguish, we also plot the difference between the two plots. We also note that the maximal difference between the two curves goes roughly like $2 \Delta$. In fact, the max difference is bounded above by $2 \Delta$ for all plots except when $\sigma = 9.4$ (red). We repeat this also for the third row in \cref{tab:de_et_al} using the privacy parameters from Table 17 from \citet{de2022unlockinghighaccuracydifferentiallyprivate}, which we copy here as \cref{app:tab_imm} and report in \cref{app:fig2}.

\begin{table}[h]
    \centering
    \caption{Hyper-parameters for training without extra data on CIFAR-10 
             with a WRN-40-4.}\label{app:tab_cifar}
    \begin{tabular}{l c c c c c c}
        \hline
        \textbf{Hyper-parameter} & 1.0 & 2.0 & 3.0 & 4.0 & 6.0 & 8.0 \\ 
        \hline
        $\varepsilon$ & 1.0 & 2.0 & 3.0 & 4.0 & 6.0 & 8.0 \\ 
        $\delta$ & $10^{-5}$ & $10^{-5}$ & $10^{-5}$ & $10^{-5}$ & $10^{-5}$ & $10^{-5}$ \\ 
        \hline
        Augmult & 32 & 32 & 32 & 32 & 32 & 32 \\ 
        Batch-size & 16384 & 16384 & 16384 & 16384 & 16384 & 16384 \\ 
        Clipping-norm & 1 & 1 & 1 & 1 & 1 & 1 \\ 
        Learning-rate & 2 & 2 & 2 & 2 & 2 & 2 \\ 
        Noise multiplier $\sigma$ & 40.0 & 24.0 & 20.0 & 16.0 & 12.0 & 9.4 \\ 
        Number Updates & 906 & 1156 & 1656 & 1765 & 2007 & 2000 \\ 
        \hline
    \end{tabular}
\end{table}

\begin{figure*}[h]
    \centering
    \begin{minipage}{0.48\linewidth}
        \centering
        \includegraphics[width=\linewidth]{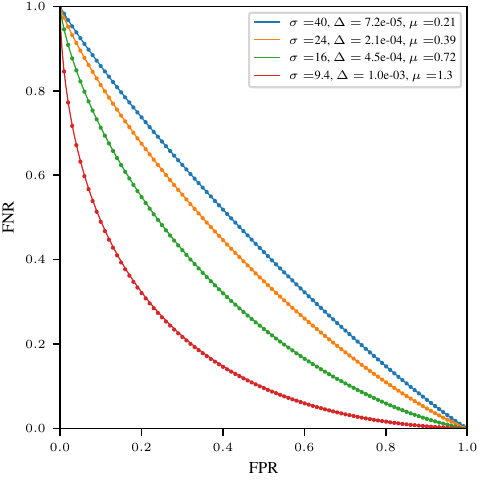}
    \end{minipage}%
    \hfill
    \begin{minipage}{0.48\linewidth}
        \centering

        \includegraphics[width=\linewidth]{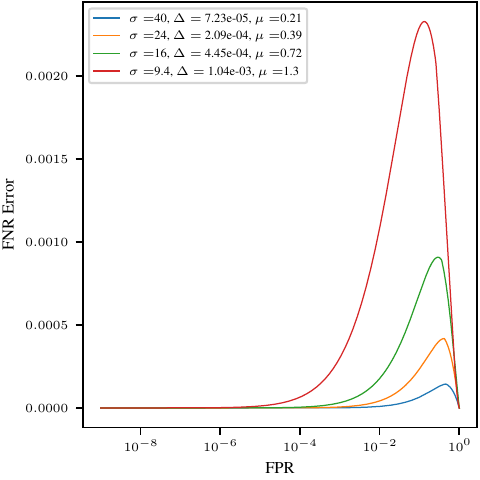}
    \end{minipage}
     \caption{Trade-off curves from the first row of \Cref{tab:de_et_al}. The dotted line refers to the $\mu$-GDP trade-off curve, and the solid line refers to the original trade-off curve from a numerical accountant. The right figure represents the difference between the dotted line and the solid lines in the left hand figure, on a logarithmic $x$ scale to emphasize small FPR.}
    \label{app:fig_1}
\end{figure*}

\begin{table}[h]
    \centering
    \caption{Hyper-parameters for ImageNet-32 $\rightarrow$ CIFAR-10, fine-tuning 
             the last layer of WRN-28-10}
    \label{app:tab_imm}
    \begin{tabular}{l c c c c}
        \hline
        \textbf{Hyper-parameter} & 1.0 & 2.0 & 4.0 & 8.0 \\
        \hline
        $\varepsilon$ & 1.0 & 2.0 & 4.0 & 8.0 \\
        $\delta$ & $10^{-5}$ & $10^{-5}$ & $10^{-5}$ & $10^{-5}$ \\
        \hline
        Augmentation multiplicity & 16 & 16 & 16 & 16 \\
        Batch-size & 16384 & 16384 & 16384 & 16384 \\
        Clipping-norm & 1 & 1 & 1 & 1 \\
        Learning-rate & 4 & 4 & 4 & 4 \\
        Noise multiplier $\sigma$ & 21.1 & 15.8 & 12.0 & 9.4 \\
        Number of updates & 250 & 500 & 1000 & 2000 \\
        \hline
    \end{tabular}
\end{table}

\begin{figure*}[h]
    \centering
    \begin{minipage}{0.48\linewidth}
        \centering
        \includegraphics[width=\linewidth]{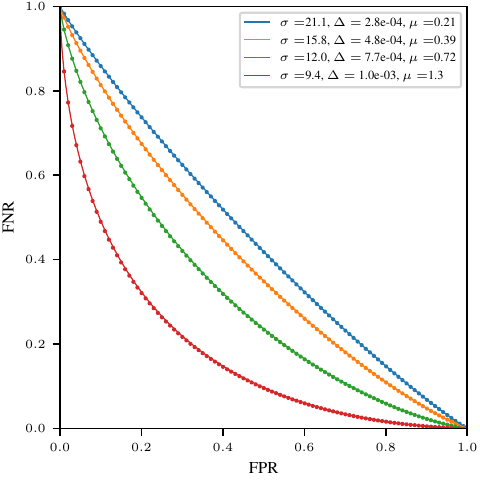}
    \end{minipage}%
    \hfill
    \begin{minipage}{0.48\linewidth}
        \centering

        \includegraphics[width=\linewidth]{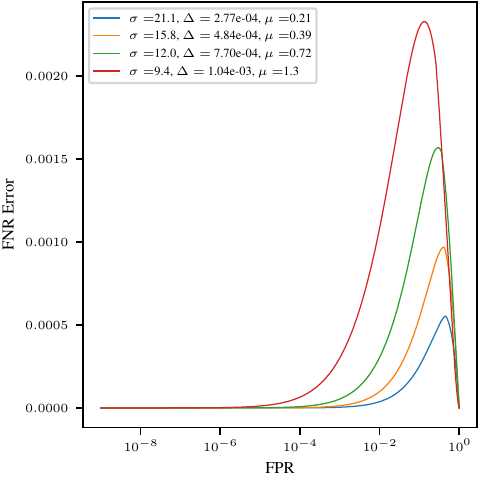}
    \end{minipage}
     \caption{Trade-off curves from the third row of \Cref{tab:de_et_al}. See the caption of \cref{app:fig_1} for details.}
    \label{app:fig2}
\end{figure*}

\end{document}